\newif\ifJOURNAL
\JOURNALfalse
\newif\ifCONF    
\CONFfalse       
\newif\ifarXiv
\arXivfalse
\newif\ifWP
\WPfalse
\newif\ifFULL
\FULLfalse

\newif\ifLandscape   
\Landscapefalse
\newif\ifPortrait    
\Portraitfalse

\WPtrue

\Portraittrue      

\newif\ifnotCONF  
\notCONFtrue
\ifCONF\notCONFfalse\fi

\newif\ifnotarXiv  
\notarXivtrue
\ifarXiv\notarXivfalse\fi

\newif\ifTR  
\TRfalse
\ifarXiv\TRtrue\fi
\ifWP\TRtrue\fi
\newif\ifnotTR
\notTRtrue
\ifarXiv\notTRfalse\fi
\ifWP\notTRfalse\fi

\ifJOURNAL

  \newcommand{\OCMVII}{Vovk/Petej:2014UAI}

  \newcommand{\OCMXII}{OCM12}
\fi
\ifCONF

  \newcommand{\OCMVII}{Vovk/Petej:2014UAI}

  \newcommand{\OCMXII}{OCM12}
\fi
\ifarXiv

  \newcommand{\OCMVII}{Vovk/Petej:arXiv1211}

  \newcommand{\OCMXII}{OCM12}
\fi
\ifWP

  \newcommand{\OCMVII}{OCM7}

  \newcommand{\OCMXII}{OCM12}
\fi

\ifCONF
  \documentclass{article}
  \usepackage{nips15submit_e,times}
  \usepackage{hyperref}
  \usepackage{url}
  \usepackage{amsmath,amsthm,amsfonts,amssymb,latexsym,graphicx,url,stmaryrd,bm}
  \usepackage{algorithm}
  \usepackage[noend]{algpseudocode}
  \newcommand{\Extra}[1]{}
\fi

\ifarXiv
  \documentclass[10pt]{article}
  \usepackage{amsmath,amsthm,amsfonts,amssymb,latexsym,graphicx,url,stmaryrd,bm}
  \usepackage{algorithm}
  \usepackage[noend]{algpseudocode}
  \newcommand{\Extra}[1]{}
\fi

\ifWP
  \documentclass[10pt]{article}
  \usepackage{amsmath,amsthm,amsfonts,amssymb,latexsym,graphicx,url,stmaryrd}
  \usepackage{algorithm}
  \usepackage[noend]{algpseudocode}
  \usepackage[T2A]{fontenc}
  \usepackage[utf8]{inputenc}

\makeatletter

\newif\iftwodates
\twodatesfalse

\renewcommand\maketitle{\begin{titlepage}%
  \let\footnotesize\small
  \let\footnoterule\relax
  \let \footnote \thanks
  \null\vfil
  \vskip 40\p@
  \begin{center}%
    {\LARGE \bf \@title \par}%
    \vskip 3em%
    {\large
     \lineskip .75em%
      \begin{tabular}[t]{c}%
        \@author
      \end{tabular}\par}%
      \vskip 1.5em%
  \end{center}\par
  \vfill
  \begin{center}
    \includegraphics[width=0.2\textwidth]{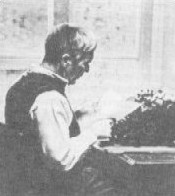}
    \hskip 3em%
    \raisebox{-0.2em}{\parbox[b]{0.6\textwidth}
    {практические выводы\\
      \hspace*{1em}
      теории вероятностей\\
      \hspace*{2em}
      могут быть обоснованы\\
      \hspace*{3em}
      в качестве следствий\\
      \hspace*{4em}
      гипотез о \emph{предельной}\\
      \hspace*{4em}
      при данных ограничениях\\
      \hspace*{3em}
      сложности изучаемых явлений}}
  \end{center}
  \@thanks
  \vfill
  \begin{center}
    {\large \bf On-line Compression Modelling Project (New Series)}
  \end{center}
  \begin{center}
    {\large Working Paper \#\No}
  \end{center}
  \begin{center}
    {\iftwodates\large First posted \firstposted.
    Last revised \@date.\else\large\@date\fi}
  \end{center}
  \begin{center}
    Project web site:\\
    http://alrw.net
  \end{center}
  \end{titlepage}%
  \setcounter{footnote}{0}%
  \global\let\thanks\relax
  \global\let\maketitle\relax
  \global\let\@thanks\@empty
  \global\let\@author\@empty
  \global\let\@date\@empty
  \global\let\@title\@empty
  \global\let\title\relax
  \global\let\author\relax
  \global\let\date\relax
  \global\let\and\relax
}

\renewenvironment{abstract}{%
  \titlepage
  \null\vfil
  \@beginparpenalty\@lowpenalty
  \begin{center}%
    \Large \bfseries \abstractname
    \@endparpenalty\@M
  \end{center}}%
  {\par\vfill\tableofcontents\endtitlepage}

\renewenvironment{thebibliography}[1]
  {\section*{\refname}%
  \addcontentsline{toc}{section}{\refname}
  \@mkboth{\MakeUppercase\refname}{\MakeUppercase\refname}%
  \list{\@biblabel{\@arabic\c@enumiv}}%
    {\settowidth\labelwidth{\@biblabel{#1}}%
    \leftmargin\labelwidth
    \advance\leftmargin\labelsep
    \@openbib@code
    \usecounter{enumiv}%
    \let\p@enumiv\@empty
    \renewcommand\theenumiv{\@arabic\c@enumiv}}%
    \sloppy
    \clubpenalty4000
    \@clubpenalty \clubpenalty
    \widowpenalty4000%
    \sfcode`\.\@m}
    {\def\@noitemerr
    {\@latex@warning{Empty `thebibliography' environment}}%
  \endlist}

\makeatother

  \newcommand{\Extra}[1]{}
  \newcommand{\zzrelax}[1]{\relax}
\fi

\ifFULL
  \usepackage{color}
  \renewcommand{\Extra}[1]{\blue{#1}}
  
  \newcommand{\blue}[1]{\textcolor{blue}{#1}}
  \newcommand{\bluebegin}{\begingroup\color{blue}}
  \newcommand{\blueend}{\endgroup}

\fi

\allowdisplaybreaks[4]

\DeclareMathOperator{\Prob}{\mathbb{P}}   
\DeclareMathOperator{\Expect}{\mathbb{E}} 

\DeclareMathOperator{\IVAP}{IVAP}        
\DeclareMathOperator{\CVAP}{CVAP}        
\DeclareMathOperator{\GM}{GM}            
\DeclareMathOperator{\key}{key}          
\DeclareMathOperator{\payload}{payload}  
\DeclareMathOperator{\BST}{BST}          

\DeclareMathOperator{\III}{\boldsymbol{1}}      

\newcommand{\KL}{\textrm{KL}}       

\newcommand{\Br}{{\rm Br}}
\newcommand{\MLL}{\textrm{MLL}}
\newcommand{\MBL}{\textrm{MBL}}

\newcommand{\dd}{\mathrm{d}}

\newcommand{\Push}{\textsc{Push}}
\newcommand{\Pop}{\textsc{Pop}}
\newcommand{\StackEmpty}{\textsc{Stack-Empty}}
\newcommand{\Top}{\textsc{Top}}
\newcommand{\NextToTop}{\textsc{Next-To-Top}}

  \newtheorem{lemma}{Lemma}
  \newtheorem{proposition}{Proposition}

  \theoremstyle{definition}
  \newtheorem*{remark}{Remark}

\algrenewcommand\algorithmicthen{\relax}
\algrenewcommand\algorithmicdo{\relax}
\algrenewcommand\algorithmiccomment{\quad$\boldsymbol{/\!/}$\quad}  
\newcommand{\Continue}{\textbf{continue}}

\newcommand{\II}{\hspace{\algorithmicindent}}  

\ifCONF
  \title{Large-scale probabilistic predictors with and without guarantees of validity}

  \newcounter{Star}
  \newcounter{Dagger}
  \author{Vladimir Vovk${}^{\fnsymbol{Star}}$,
    Ivan Petej${}^{\fnsymbol{Star}}$,
    and Valentina Fedorova${}^{\fnsymbol{Dagger}}$\\
  ${}^{\fnsymbol{Star}}$Department of Computer Science,
  Royal Holloway, University of London, UK\\
  ${}^{\fnsymbol{Dagger}}$Yandex, Moscow, Russia\\
  \texttt{\{volodya.vovk,ivan.petej,alushaf\}@gmail.com}}

  \nipsfinalcopy 
\fi

\ifarXiv
  \title{Large-scale probabilistic prediction with and without validity guarantees\thanks{The conference version of this paper
    is to appear in \emph{Advances in Neural Information Processing Systems 28}, 2015.}}
  \author{Vladimir Vovk, Ivan Petej, and Valentina Fedorova\\
    \texttt{\{volodya.vovk,ivan.petej,alushaf\}{\rm@}gmail.com}}
\fi

\ifWP
  \title{Large-scale probabilistic prediction with and without validity guarantees}
  \author{Vladimir Vovk, Ivan Petej, and Valentina Fedorova}
  \newcommand{\No}{13}
\fi

\begin{document}
  \maketitle


\begin{abstract}
  This paper studies theoretically and empirically
  a method of turning machine-learning algorithms
  into probabilistic predictors that
  automatically enjoys a property of validity (perfect calibration)
  and is computationally efficient.
  The price to pay for perfect calibration is that these probabilistic predictors
  produce imprecise (in practice, almost precise for large data sets) probabilities.
  When these imprecise probabilities are merged into precise probabilities,
  the resulting predictors, while losing the theoretical property of perfect calibration,
  are consistently more accurate than the existing methods in empirical studies.
  \ifWP

    \bigskip

    \noindent
    The conference version of this paper
    is to appear in \emph{Advances in Neural Information Processing Systems 28}, 2015.
  \fi
\end{abstract}

\section{Introduction}

Prediction algorithms studied in this paper belong to the class of Venn--Abers predictors,
introduced in \cite{\OCMVII}.
They are based on the method of isotonic regression \cite{Ayer/etal:1955}
and prompted by the observation that when applied in machine learning
the method of isotonic regression
often produces miscalibrated probability predictions
(see, e.g., \cite{Jiang/etal:2011,Lambrou/etal:2012});
it has also been reported (\cite{Caruana/NM:2006}, Section~1)
that isotonic regression is more prone to overfitting than Platt's scaling
\cite{Platt:2000} when data is scarce.
The advantage of Venn--Abers predictors is that they are a special case of Venn predictors
(\cite{Vovk/etal:2005book}, Chapter~6),
and so (\cite{Vovk/etal:2005book}, Theorem~6.6) are always well-calibrated
(cf.\ Proposition~\ref{prop:validity} below).
They can be considered to be a regularized version of the procedure
used by \cite{Zadrozny/Elkan:2001}, which helps them resist overfitting.

The main desiderata for Venn (and related conformal, \cite{Vovk/etal:2005book}, Chapter~2) predictors
are validity, predictive efficiency, and computational efficiency.
This paper introduces two computationally efficient versions of Venn--Abers predictors,
which we refer to as inductive Venn--Abers predictors (IVAPs) and cross-Venn--Abers predictors (CVAPs).
The ways in which they achieve the three desiderata are:
\begin{itemize}
\item
  Validity (in the form of perfect calibration) is satisfied by IVAPs automatically,
  and the experimental results reported in this paper suggest that it is inherited by CVAPs.
\item
  Predictive efficiency is determined by the predictive efficiency
  of the underlying learning algorithms
  (so that the full arsenal of methods of modern machine learning
  can be brought to bear on the prediction problem at hand).
\item
  Computational efficiency is, again, determined by the computational efficiency
  of the underlying algorithm;
  the computational overhead of extracting probabilistic predictions consists of sorting
  (which takes time $O(n\log n)$, where $n$ is the number of observations)
  and other computations taking time $O(n)$.
\end{itemize}
An advantage of Venn prediction over conformal prediction,
which also enjoys validity guarantees,
is that Venn predictors output probabilities rather than p-values,
and probabilities, in the spirit of Bayesian decision theory,
can be easily combined with utilities to produce optimal decisions.

In Sections~\ref{sec:IVAP} and~\ref{sec:CVAP} we discuss IVAPs and CVAPs, respectively.
Section~\ref{sec:merging} is devoted to minimax ways of merging imprecise probabilities
into precise probabilities
and thus making IVAPs and CVAPs precise probabilistic predictors.

In 
this paper we concentrate on binary classification problems,
in which the objects to be classified are labelled as 0 or 1.
Most of machine learning algorithms are \emph{scoring algorithms},
in that they output a real-valued score for each test object,
which is then compared with a threshold to arrive at a categorical prediction, 0 or 1.
As precise probabilistic predictors, IVAPs and CVAPs
are ways of converting the scores for test objects into numbers in the range $[0,1]$
that can serve as probabilities,
or \emph{calibrating} the scores.
In Section~\ref{sec:comparison} we\ discuss two existing calibration methods,
Platt's \cite{Platt:2000} and the method \cite{Zadrozny/Elkan:2001} based on isotonic regression\ifnotCONF,
  and compare them with IVAPs and CVAPs theoretically\fi.
Section~\ref{sec:experiments} is devoted to experimental comparisons
and shows that CVAPs consistently outperform the two existing methods.


\section{Inductive Venn--Abers predictors (IVAPs)}
\label{sec:IVAP}

In this paper we consider data sequences (usually loosely referred to as sets)
consisting of \emph{observations} $z=(x,y)$,
each observation consisting of an \emph{object} $x$ and a \emph{label} $y\in\{0,1\}$;
we only consider binary labels.
We are given a training set whose size will be denoted $l$.

This section introduces inductive Venn--Abers predictors.
Our main concern is how to implement them efficiently,
but as functions, an IVAP is defined in terms of a scoring algorithm
(see the last paragraph of the previous section) as follows:
\begin{itemize}
\item
  Divide the training set of size $l$ into two subsets,
  the \emph{proper training set} of size $m$ and the \emph{calibration set} of size $k$,
  so that $l=m+k$.
\item
  Train the scoring algorithm on the proper training set.
\item
  Find the scores $s_1,\ldots,s_k$ of the calibration objects $x_1,\ldots,x_k$.
\item
  When a new test object $x$ arrives, compute its score $s$.
  Fit isotonic regression to $(s_1,y_1),\ldots,(s_k,y_k),(s,0)$ obtaining a function $f_0$.
  Fit isotonic regression to $(s_1,y_1),\ldots,(s_k,y_k),(s,1)$ obtaining a function $f_1$.
  The multiprobability prediction for the label $y$ of $x$
  is the pair $(p_0,p_1):=(f_0(s),f_1(s))$
  (intuitively, the prediction is that the probability that $y=1$ is either $f_0(s)$ or $f_1(s)$).
\end{itemize}

Notice that the multiprobability prediction $(p_0,p_1)$ output by an IVAP
always satisfies $p_0<p_1$,
and so $p_0$ and $p_1$ can be interpreted as the lower and upper probabilities,
respectively;
in practice, they are close to each other for large training sets.

First we state formally the property of validity of IVAPs
(adapting the approach of \cite{\OCMVII} to IVAPs).
A random variable $P$ taking values in $[0,1]$ is \emph{perfectly calibrated}
(as a predictor) for a random variable $Y$ taking values in $\{0,1\}$
if $\Expect(Y\mid P) = P$ a.s.
A \emph{selector} is a random variable taking values in $\{0,1\}$.
As a general rule,
in this paper
random variables are denoted by capital letters
(e.g., $X$ are random objects and $Y$ are random labels).

\begin{proposition}\label{prop:validity}
  Let $(P_0,P_1)$ be an IVAP's prediction for $X$
  based on a training sequence $(X_1,Y_1),\ldots,(X_l,Y_l)$.
  There is a selector $S$ such that $P_S$ is perfectly calibrated for $Y$
  provided the random observations $(X_1,Y_1),\ldots,(X_l,Y_l),(X,Y)$ are i.i.d.
\end{proposition}

Our next proposition concerns the computational efficiency of IVAPs;\ifnotCONF\
  both propositions will be proved later in the section\fi.
\begin{proposition}\label{prop:computational-efficiency}
  Given the scores $s_1,\ldots,s_k$ of the calibration objects,
  the prediction rule for computing the IVAP's predictions
  can be computed in time $O(k\log k)$ and space $O(k)$.
  Its application to each test object takes time $O(\log k)$.
  Given the sorted scores of the calibration objects,
  the prediction rule can be computed in time and space $O(k)$.
\end{proposition}

Proofs of both statements rely on the geometric representation
of isotonic regression as the slope of the GCM (greatest convex minorant)
of the CSD (cumulative sum diagram):
see \cite{Barlow/etal:1972}, pages 9--13 (especially Theorem~1.1).
To make our exposition more self-contained, we define both GCM and CSD below.

First we explain how to fit isotonic regression to $(s_1,y_1),\ldots,(s_k,y_k)$
(without necessarily assuming that $s_i$ are the calibration scores and $y_i$ are the calibration labels,
which will be needed to cover the use of isotonic regression in IVAPs).
We start from sorting all scores $s_1,\ldots,s_{k}$ in the increasing order
and removing the duplicates.
(This is the most computationally expensive step in our calibration procedure,
$O(k\log k)$ in the worst case.)
Let $k'\le k$ be the number of distinct elements among $s_1,\ldots,s_k$,
i.e., the cardinality of the set $\{s_1,\ldots,s_k\}$.
Define $s'_j$, $j=1,\ldots,k'$, to be the $j$th smallest element of $\{s_1,\ldots,s_k\}$,
so that $s'_1<s'_2<\cdots<s'_{k'}$.
Define $w_j:=\left|\left\{i=1,\ldots,k:s_i=s'_j\right\}\right|$
to be the number of times $s'_j$ occurs among $s_1,\ldots,s_k$.
Finally, define
$$
  y'_j
  :=
  \frac{1}{w_j}
  \sum_{i=1,\ldots,k:s_i=s'_j}
  y_i
$$
to be the average label corresponding to $s_i=s'_j$.



The \emph{CSD} of $(s_1,y_1),\ldots,(s_k,y_k)$ is the set of points
\begin{equation}\label{eq:CSD}
  P_i
  :=
  \left(
    \sum_{j=1}^i
    w_j,
    \sum_{j=1}^i
    y'_j w_j
  \right),
  \quad
  i=0,1,\ldots,k';
\end{equation}
in particular, $P_0=(0,0)$.
The \emph{GCM} is the greatest convex minorant of the CSD.
The value at $s'_i$, $i=1,\ldots,k'$, of the \emph{isotonic regression}
fitted to $(s_1,y_1),\ldots,(s_k,y_k)$ is defined to be the slope of the GCM
between $\sum_{j=1}^{i-1}w_j$ and $\sum_{j=1}^i w_j$;
the values at other $s$ are somewhat arbitrary
(namely, the value at $s\in(s'_{i},s'_{i+1})$ can be set to anything
between the left and right slopes of the GCM at $\sum_{j=1}^i w_j$)
but are never needed in this paper
(unlike in the standard use of isotonic regression in machine learning, \cite{Zadrozny/Elkan:2001}):
e.g., $f_1(s)$ is the value of the isotonic regression fitted to a sequence that already contains $(s,1)$.

\begin{proof}[Proof of Proposition~\ref{prop:validity}]
  Set $S:=Y$.
  The statement of the proposition even holds conditionally
  on knowing the values of $(X_1,Y_1),\ldots,(X_m,Y_m)$
  and the multiset $\lbag(X_{m+1},Y_{m+1}),\ldots,(X_l,Y_l)$,  
  $(X,Y)\rbag$;
  this knowledge allows us to compute the scores $\lbag s_1,\ldots,s_k,s\rbag$
  of the calibration objects $X_{m+1},\ldots,X_l$ and the test object $X$.
  The only remaining randomness
  is over the equiprobable permutations of $(X_{m+1},Y_{m+1}),\ldots,(X_l,Y_l),(X,Y)$;
  in particular, $(s,Y)$ is drawn randomly
  from the multiset $\lbag(s_{1},Y_{m+1}),\ldots,(s_k,Y_l),(s,Y)\rbag$.
  It remains to notice that, according to the GCM construction,
  the average label of the calibration and test observations
  corresponding to a given value of $P_S$ is equal to $P_S$.
  \Extra{The last observation is stated, in a more general case,
    as Theorem~1.3.5 in \cite{Robertson/etal:1988}.}
\end{proof}

The idea behind computing the pair $(f_0(s),f_1(s))$ efficiently
is to pre-compute two vectors $F^0$ and $F^1$ storing $f_0(s)$ and $f_1(s)$,
respectively,
for all possible values of $s$.
Let $k'$ and $s'_i$ be as defined above
in the case where $s_1,\ldots,s_k$ are the calibration scores
and $y_1,\ldots,y_k$ are the corresponding labels.
The vectors $F^0$ and $F^1$ are of length $k'$,
and for all $i=1,\ldots,k'$ and both~$\epsilon\in\{0,1\}$,
$F^{\epsilon}_{i}$ is the value of $f_{\epsilon}(s)$ when $s=s'_i$.
Therefore, for all $i=1,\ldots,k'$:
\begin{itemize}
\item
  $F^{1}_{i}$ is also the value of $f_{1}(s)$ when $s$ is just to the left of $s'_i$;
\item
  $F^{0}_{i}$ is also the value of $f_{0}(s)$ when $s$ is just to the right of $s'_i$.
\end{itemize}
Since $f_0$ and $f_1$ can change their values only at the points $s'_i$,
the vectors $F^0$ and $F^1$ uniquely determine the functions $f_0$ and $f_1$,
respectively.

\begin{remark}
  There are several algorithms
  for performing isotonic regression on a partially, rather than linearly, ordered set:
  see, e.g., \cite{Barlow/etal:1972}, Section~2.3
  (although one of the algorithms described in that section, the Minimax Order Algorithm,
  was later shown to be defective \cite{Lee:1983,Murray:1983}).
  Therefore, IVAPs (and CVAPs below) can be defined in the situation
  where scores take values only in a partially ordered set;
  moreover, Proposition~\ref{prop:validity} will continue to hold.
  (For the reader familiar with the notion of Venn predictors
  we could also add that Venn--Abers predictors will continue to be Venn predictors,
  which follows from the isotonic regression being the average of the original function
  over certain equivalence classes.)
  The importance of partially ordered scores stems from the fact that they enable us
  to benefit from a possible ``synergy'' between two or more prediction algorithms
  \cite{Vapnik:2015Yandex}.
  Suppose, e.g., that one prediction algorithm outputs (scalar) scores $s^1_1,\ldots,s^1_k$ for the calibration objects $x_1,\ldots,x_k$ 
  and another outputs $s^2_1,\ldots,s^2_k$ for the same calibration objects;
  we would like to use both sets of scores.
  We could merge the two sets of scores into composite vector scores,
  $s_i:=(s^1_i,s^2_i)$, $i=1,\ldots,k$,
  and then classify a new object $x$ as described earlier using its composite score $s:=(s^1,s^2)$,
  where $s^1$ and $s^2$ are the scalar scores computed by the two algorithms
  and the partial order between composite scores is defined as usual,
  $$
    (s^1,s^2)\preceq(t^1,t^2)
    \Longleftrightarrow
    (s^1\le t^1)\;\&\;(s^2\le t^2).
  $$
  Preliminary results reported in \cite{Vapnik:2015Yandex} in a related context
  suggest that the resulting predictor can outperform predictors based on the individual scalar scores.
  However, we will not pursue this idea further in this paper.
\end{remark}

\ifnotCONF
\subsection*{Computational details of IVAPs}

Let $k'$, $s'_i$, and $w_i$ be as defined above
in the case where $s_1,\ldots,s_k$ and $y_1,\ldots,y_k$ are the calibration scores and labels.
The \emph{corners} of a GCM are the points on the GCM where the slope of the GCM changes.
It is clear that the corners belong to the CSD,
and we also add the extreme points ($P_0$ and $P_{k'}$ in the case of \eqref{eq:CSD})
of the CSD to the list of corners.

We will only explain in detail how to compute $F^1$;
the computation of $F^0$ is analogous and will be explained only briefly.
First we explain how to compute~$F^1_1$.

Extend the CSD as defined above
(in the case where $s_1,\ldots,s_k$ and $y_1,\ldots,y_k$ are the calibration scores and labels)
by adding the point $P_{-1}:=(-1,-1)$.
The corresponding GCM will be referred to as the \emph{initial GCM};
it has at most $k'+2$ corners.
Algorithm~\ref{alg:preprocessing-F1}, which operates with a stack $S$ (initially empty),
computes the corners;
it is a trivial modification of Graham's scan
(\cite{Graham:1972}; \cite{Cormen/etal:2009}, Section~33.3).
The corners are returned on the stack $S$,
and they are ordered from left to right
($P_{-1}$ being at the bottom of $S$ and $P_{k'}$ at the top).
The operator ``and'' in line~\ref{ln:and} is, as usual, short circuiting.
The expression ``the angle formed by points $a$, $b$, and $c$
makes a nonleft (resp.\ nonright) turn''
may be taken to mean that $(b-a)\times(c-b)\le0$ (resp.\ ${}\ge0$),
where $\times$ stands for cross product of planar vectors;
this avoids computing angles and divisions
(see, e.g., \cite{Cormen/etal:2009}, Section~33.1).

\begin{algorithm}[b]  
  \caption{Initializing the corners for computing $F^1$}
  \label{alg:preprocessing-F1}
  \begin{algorithmic}[1]
    \State $\Push(P_{-1},S)$
    \State $\Push(P_{0},S)$
    \For{$i\in\{1,2,\ldots,k'\}$}
      \While{$S.\text{size}>1$ and the angle formed by points
          \Statex\II\II\II $\NextToTop(S)$,  
          $\Top(S)$, and $P_i$
          \Statex\II\II\II makes a nonleft turn}\label{ln:and}
        \State $\Pop(S)$
      \EndWhile
      \State $\Push(P_i,S)$
    \EndFor
    \State\Return $S$
  \end{algorithmic}
\end{algorithm}

Algorithm~\ref{alg:preprocessing-F1} allows us to compute $F^1_1$
as the slope of the line between the two bottom corners in $S$,
but this will be done by the next algorithm.



\begin{algorithm}[bt]
  \caption{Computing $F^1$}
  \label{alg:bulk-F1}
  \begin{algorithmic}[1]
    \While{$\lnot\StackEmpty(S)$}\label{ln:start-copying}
      \State $\Push(\Pop(S),S')$\label{ln:end-copying}
    \EndWhile
    \For{$i\in\{1,2,\ldots,k'\}$}\label{ln:start-loop}
      \State set $F^1_{i}$ to the slope of  
          \Statex\II\II\II $\overrightarrow{\Top(S'),\NextToTop(S')}$\label{ln:report}
      \State $P_{i-1} = P_{i-2} + P_{i} - P_{i-1}$\label{ln:reflect}
      \If{$P_{i-1}$ is at or above $\overrightarrow{\Top(S'),\NextToTop(S')}$}
        \State\Continue
      \EndIf
      \State $\Pop(S')$
      \While{$S'.\text{size}>1$ and the angle formed by points  
          \Statex\II\II\II  $P_{i-1}$, $\Top(S')$, and $\NextToTop(S')$
          \Statex\II\II\II makes a nonleft turn}
        \State $\Pop(S')$
      \EndWhile
      \State $\Push(P_{i-1},S')$
    \EndFor\label{ln:end-loop}
    \State\Return $F^1$
  \end{algorithmic}
\end{algorithm}

The rest of the procedure for computing the vector $F^1$
is shown as Algorithm~\ref{alg:bulk-F1}.
The main data structure in Algorithm~\ref{alg:bulk-F1} is a stack $S'$,
which is initialized (in lines~\ref{ln:start-copying}--\ref{ln:end-copying})
by putting in it all corners of the initial GCM
in reverse order as compared with $S$
(so that $P_{-1}=(-1,-1)$ is initially at the top of $S'$).

At each point in the execution of Algorithm~\ref{alg:bulk-F1}
we will have a length-1 \emph{active interval}
and the \emph{active corner},
which will nearly always be at the top of the stack $S'$.
The initial CSD can be visualized by connecting each pair of adjacent points:
$P_{-1}$ and $P_0$, $P_0$ and $P_1$, etc.
It stretches over the interval $[-1,k']$ of the horizontal axis;
the subinterval $[-1,0]$ corresponds to the test score $s$
(assumed to be to the left of all $s'_i$)
and each subinterval $\left[\sum_{j=1}^{i-1}w_j,\sum_{j=1}^{i}w_j\right]$
corresponds to the calibration score $s'_i$,
$i=1,\ldots,k'$.
The active corner is initially at $P_{-1}=(-1,-1)$;
the corners to the left of the active corner are irrelevant and ignored
(not remembered in $S'$).
The active interval is always between the first coordinate of $\Top(S')$
and the first coordinate of $\NextToTop(S')$.
At each iteration $i=1,\ldots,k'$ of the main loop \ref{ln:start-loop}--\ref{ln:end-loop}
we are computing $F^1_{i}$, i.e., $f_1(s)$ for the situation
where $s$ is between $s'_{i-1}$ and $s'_{i}$ (meaning to the left of $s'_1$ if $i=1$),
and after that we swap the active interval (corresponding to $s$)
and the interval corresponding to $s'_i$;
of course, after swapping pieces of CSD are adjusted vertically
in order to make the CSD as a whole continuous.

At the beginning of each iteration $i$ of the loop \ref{ln:start-loop}--\ref{ln:end-loop}
we have the CSD
\begin{equation}\label{eq:start-CSD}
  P_{-1},P_0,P_1,\ldots,P_{k'}
\end{equation}
corresponding to
\begin{align*}
  \text{the points } & s'_1,\ldots, s'_{i-1},s,s'_i,s'_{i+1},\ldots,s'_{k'}\\
  \text{with the weights } & w_1,\ldots,w_{i-1},1,w_i,w_{i+1},\ldots,w_{k'} 
\end{align*}
(respectively);  
the active interval is the projection of $\overrightarrow{P_{i-2},P_{i-1}}$
(onto the horizontal axis, here and later).
At the end of that iteration
we have the CSD which looks identical to \eqref{eq:start-CSD}
but in fact contains a different point $P_{i-1}$ (cf.\ line~\ref{ln:reflect} of the algorithm)
and corresponds to
\begin{align*}
  \text{the points } & s'_1,\ldots,s'_{i-1},s'_i,s,s'_{i+1},\ldots,s'_{k'}\\
  \text{with the weights } & w_1,\ldots,w_{i-1},w_i,1,w_{i+1},\ldots,w_{k'} 
\end{align*}
(respectively);  
the active interval becomes the projection of $\overrightarrow{P_{i-1},P_{i}}$.
To achieve this,
in line~\ref{ln:reflect} we redefine $P_{i-1}$ to be the reflection of the old $P_{i-1}$
across the mid-point $(P_{i-2}+P_{i})/2$.
The stack $S'$ always consists of corners of the GCM of the current CSD,
and it contains all the corners to the right of the active interval
(plus one more corner, which is the active corner).

At each iteration $i$ of the loop \ref{ln:start-loop}--\ref{ln:end-loop}:
\begin{itemize}
\item
  We report the slope of the GCM over the active interval as $F^1_{i}$
  (line~\ref{ln:report}).
\item
  We then swap the fragments of the CSD corresponding to the active interval and to $s'_i$
  leaving the rest of the CSD intact.
  This way the active interval moves to the right
  (from the projection of $\overrightarrow{P_{i-2},P_{i-1}}$
  to the projection of $\overrightarrow{P_{i-1},P_{i}}$).
\item
  If the point $P_{i-1}$ above the left end-point of the active interval
  is above (or at) the GCM,
  move to the next iteration of the loop.
  (The active corner does not change.)
  The rest of this description assumes that $P_{i-1}$ is strictly below.
\item
  Make $P_{i-1}$
  the active corner.
  Redefine the GCM to the right of the active corner
  by connecting the active corner
  to the right-most corner $C$ such that the slope of the line connecting the active corner and that corner is minimal;
  all the corners between the active corner and that right-most corner $C$
  are then forgotten.
\end{itemize}  

\begin{lemma}\label{lem:time}
  The worst-case computation time of Algorithms~\ref{alg:preprocessing-F1} and~\ref{alg:bulk-F1}
  is $O(k')$.
\end{lemma}

\begin{proof}
  In the case of Algorithm~\ref{alg:preprocessing-F1},
  see \cite{Cormen/etal:2009}, Section~33.3.
  In the case of Algorithm~\ref{alg:bulk-F1},
  it suffices to notice that the total number of iterations for the \textbf{while} loop
  does not exceed the total number of elements pushed onto $S'$
  (since at each iteration we pop an element off $S'$);
  and the total number of elements pushed onto $S'$
  is at most $k'$ (in the first \textbf{for} loop)
  plus $k'$ (in the second \textbf{for} loop).
\end{proof}

  For convenience of the reader wishing to program IVAPs and CVAPs,
  we also give the counterparts of Algorithms~\ref{alg:preprocessing-F1} and~\ref{alg:bulk-F1}
  for computing $F^0$:
  see Algorithms~\ref{alg:preprocessing-F0} and~\ref{alg:bulk-F0} below.
  In those algorithms,
  we do not need the point $P_{-1}$ anymore;
  however, we need a new point $P_{k'+1}:=P_{k'}+(1,1)$.
  The stacks $S$ and $S'$ that they use are initially empty.

  \begin{algorithm}[bt]
    \caption{Initializing the corners for computing $F^0$}
    \label{alg:preprocessing-F0}
    \begin{algorithmic}[1]
      \State $\Push(P_{k'+1},S)$
      \State $\Push(P_{k'},S)$
      \For{$i\in\{k'-1,k'-2,\ldots,0\}$}
        \While{$S.\text{size}>1$ and the angle formed by points $\NextToTop(S)$,
            \Statex\II\II\II $\Top(S)$, and $P_i$ makes a nonright turn}
          \State $\Pop(S)$
        \EndWhile
        \State $\Push(P_i,S)$
      \EndFor
      \State\Return $S$
    \end{algorithmic}
  \end{algorithm}

  \begin{algorithm}[bt]
    \caption{Computing $F^0$}
    \label{alg:bulk-F0}
    \begin{algorithmic}[1]
      \While{$\lnot\StackEmpty(S)$}
        \State $\Push(\Pop(S),S')$
      \EndWhile
      \For{$i\in\{k',k'-1,\ldots,1\}$}
        \State set $F^0_{i}$ to the slope of $\overrightarrow{\Top(S'),\NextToTop(S')}$
        \State $P_{i} = P_{i-1} + P_{i+1} - P_{i}$
        \If{$P_{i}$ is at or above $\overrightarrow{\Top(S'),\NextToTop(S')}$}
          \State\Continue
        \EndIf
        \State $\Pop(S')$
        \While{$S'.\text{size}>1$ and the angle formed by points $P_{i}$,
            \Statex\II\II\II $\Top(S')$, and $\NextToTop(S')$ makes a nonright turn}
          \State $\Pop(S')$
        \EndWhile
        \State $\Push(P_{i},S')$
      \EndFor
      \State\Return $F^0$
    \end{algorithmic}
  \end{algorithm}

  Alternatively, we could use the algorithm for computing $F^1$
  in order to compute $F^0$,
  since, for all $i\in\{1,\ldots,k'\}$,
  \begin{multline*}
    F^0_i
    \left(
      s'_1,\ldots,s'_{k'},
      w_1,\ldots,w_{k'},
      y'_1,\ldots,y'_{k'}
    \right)\\
    =
    1
    -
    F^1_i
    \bigl(
      -s'_1,\ldots,-s'_{k'},
      w_1,\ldots,w_{k'},    
      1-y'_1,\ldots,1-y'_{k'}
    \bigr),
  \end{multline*}
  where the dependence on various parameters is made explicit.

After computing $F^0$ and $F^1$ we can arrange the calibration scores $s'_1,\ldots,s'_{k'}$
into a binary search tree: see Algorithm~\ref{alg:BST},
where $F^0_0$ is defined to be $0$ and $F^1_{k'+1}$ is defined to be $1$;
we will refer to $s'_i$ as the \emph{keys} of the corresponding nodes
(only internal nodes will have keys).
Algorithm~\ref{alg:BST} is in fact more general than what we need:
it computes the binary search tree for the scores $s'_a,s'_{a+1},\ldots,s'_b$
for $a\le b$;
therefore, we need to run $\BST(1,k')$.
The size of the binary search tree is $2k'+1$;
$k'$ of its nodes are internal nodes corresponding to different values of $s'_i$,
$i=1,\ldots,k'$,
and the other $k'+1$ of its nodes are leaves corresponding to the $k'+1$ intervals
formed by the points $s'_1,\ldots,s'_{k'}$.

\begin{algorithm}[bt]
  \caption{$\BST(a,b)$ (to create the binary search tree, run $\BST(1,k')$)}
  \label{alg:BST}
  \begin{algorithmic}[1]
    \If{$b=a$}
      \State construct the binary tree \Statex\II whose root has key $s'_a$ and payload $\{F^0_a,F^1_a\}$,
      \Statex\II left child is a leaf with payload $\{F^0_{a-1},F^1_a\}$,
      \Statex\II and right child is a leaf with payload $\{F^0_{a},F^1_{a+1}\}$
      \State\Return its root
    \ElsIf{$b=a+1$}
      \State construct the binary tree \Statex\II whose root has key $s'_a$ and payload $\{F^0_a,F^1_a\}$,
      \Statex\II left child is a leaf with payload $\{F^0_{a-1},F^1_a\}$,
      \Statex\II and right child is $\BST(b,b)$
      \State\Return its root
    \ElsIf
      \State $c=\lfloor(a+b)/2\rfloor$
      \State construct the binary tree \Statex\II whose root has key $s'_c$ and payload $\{F^0_c,F^1_c\}$,
      \Statex\II left child is $\BST(a,c-1)$, \Statex\II and right child is $\BST(c+1,b)$
      \State\Return its root
    \EndIf
  \end{algorithmic}
\end{algorithm}

Once we have the binary search tree it is easy to compute the prediction for a test object $x$
in time logarithmic in $k'$:
see Algorithm~\ref{alg:IVAP},
which passes $x$ through the tree and uses $N$ to denote the current node.
Formally, we give the test object $x$, the proper training set $T'$, and the calibration set $T''$
as the inputs of Algorithm~\ref{alg:IVAP};
however, the algorithm uses for prediction
the binary search tree built from $T'$ and $T''$,
and the bulk of work is done in Algorithms~\ref{alg:preprocessing-F1}--\ref{alg:BST}.

\begin{algorithm}[bt]
  \caption{$\IVAP(T',T'',x)$\Comment{inductive Venn--Abers predictor}}  
  \label{alg:IVAP}
  \begin{algorithmic}[1]
    \State set $N$ to the root of the binary search tree and compute the score $s$ of $x$
    \While{$N$ is not a leaf}
      \If{$s<\key(N)$}
        \State set $N$ to $N$'s left child
      \ElsIf{$s>\key(N)$}
        \State set $N$ to $N$'s right child
      \Else\Comment{if $s=\key(N)$}
        \State\Return $\payload(N)$
      \EndIf
    \EndWhile
    \State\Return $\payload(N)$
  \end{algorithmic}
\end{algorithm}

The worst-case computational complexity of the overall procedure
involves the following components:
\begin{itemize}
\item
  Training the algorithm on the proper training set,
  computing the scores of the calibration objects,
  and computing the scores of the test objects;
  at this stage the computation time is determined by the underlying algorithm.
\item
  Sorting the scores of the calibration objects
  takes time $O(k\log k)$.
\item
  Running our procedure for pre-computing $f_0$ and $f_1$
  takes time $O(k)$
  (by Lemma~\ref{lem:time}).
\item
  Processing each test object takes an additional time of $O(\log k)$
  (using binary search).
\end{itemize}
In principle, using binary search does not require an explicit construction of a binary search tree
(cf.\ \cite{Cormen/etal:2009}, Exercise~2.3-5),
but once we have a binary search tree we can easily transform it into a red-black tree,
which allows us to add new observations to (and remove old observations from)
the calibration set in time $O(\log k)$
(\cite{Cormen/etal:2009}, Chapter~13).
\fi  

\section{Cross Venn--Abers predictors (CVAPs)}
\label{sec:CVAP}

A CVAP is just a combination of $K$ IVAPs,
where $K$ is the parameter of the algorithm.
It is described as Algorithm~\ref{alg:CVAP},
where $\IVAP(A,B,x)$ stands for the output of IVAP applied to $A$ as proper training set,
$B$ as calibration set, and $x$ as test object,
and $\GM$ stands for geometric mean
(so that $\GM(p_1)$ is the geometric mean of $p_1^1,\ldots,p_1^K$
and $\GM(1-p_0)$ is the geometric mean of $1-p_0^1,\ldots,1-p_0^K$).
The folds should be of approximately equal size,
and usually the training set is split into folds at random
(although we choose contiguous folds in Section~\ref{sec:experiments}
to facilitate reproducibility).
One way to obtain a random assignment of the training observations to folds
(see line \ref{ln:split}) is to start from a regular array
in which the first $l_1$ observations are assigned to fold~1,
the following $l_2$ observations are assigned to fold~2,
up to the last $l_K$ observations which are assigned to fold~$K$,
where $\left|l_k-l/K\right|<1$ for all $k$,
and then to apply a random permutation.
Remember that the procedure \textsc{Randomize-in-Place}
(\cite{Cormen/etal:2009}, Section~5.3)
can do the last step in time $O(l)$.
See the next section for a justification of the expression $\GM(p_1)/(\GM(1-p_0)+\GM(p_1))$
used for merging the IVAPs' outputs.

\begin{algorithm}[bt]
  \caption{$\CVAP(T,x)$\Comment{cross-Venn--Abers predictor for training set $T$}}  
  \label{alg:CVAP}
  \begin{algorithmic}[1]
    \State\label{ln:split} split the training set $T$ 
        into $K$ folds $T_1,\ldots,T_K$ 
    \For{$k\in\{1,\ldots,K\}$}
      \State $(p_0^k,p_1^k):=\IVAP(T\setminus T_k,T_k,x)$
    \EndFor
    \State\Return $\GM(p_1)/(\GM(1-p_0)+\GM(p_1))$
  \end{algorithmic}
\end{algorithm}

\section{Making probability predictions out of multiprobability ones}
\label{sec:merging}

In CVAP (Algorithm~\ref{alg:CVAP})
we merge the $K$ multiprobability predictions output by $K$ IVAPs.
In this section we design a minimax way
for merging them,
essentially following \cite{\OCMVII}.
For the log-loss function the result is especially simple,
$\GM(p_1)/(\GM(1-p_0)+\GM(p_1))$.

\ifnotCONF
\begin{remark}
  Notice that the probability interval $(1-\GM(1-p_0),\GM(p_1))$ (formally, a pair of numbers)
  is narrower than the corresponding interval for the arithmetic means;
  this follows from the fact that a geometric mean never exceeds the corresponding arithmetic mean
  and that we always have $p_0<p_1$.
\end{remark}
\fi

Let us check that $\GM(p_1)/(\GM(1-p_0)+\GM(p_1))$
is indeed the minimax expression under log loss.
Suppose the pairs of lower and upper probabilities to be merged
are $(p^1_0,p^1_1),\ldots,(p^K_0,p^K_1)$
and the merged probability is $p$.
The extra cumulative loss suffered by $p$
over the correct members $p^1_1,\ldots,p^K_1$ of the pairs
when the true label is $1$ is
\begin{equation}\label{eq:1}
  \log\frac{p^1_1}{p}+\cdots+\log\frac{p^K_1}{p},
\end{equation}
and the extra cumulative loss of $p$
over the correct members of the pairs
when the true label is $0$ is
\begin{equation}\label{eq:0}
  \log\frac{1-p^1_0}{1-p}+\cdots+\log\frac{1-p^K_0}{1-p}.
\end{equation}
Equalizing the two expressions we obtain
$$
  \frac{p^1_1 \cdots p^K_1}{p^K}
  =
  \frac{(1-p^1_0)\cdots(1-p^K_0)}{(1-p)^K},
$$
which gives the required minimax expression for the merged probability
(since \eqref{eq:1} is decreasing and \eqref{eq:0} is increasing in $p$).

\ifnotCONF
%
In the case of the Brier loss function,
we solve the linear equation
\begin{equation*}  
  (1-p)^2 - (1-p^1_{1})^2 + \cdots + (1-p)^2 - (1-p^K_{1})^2
  =
  p^2 - (p^1_{0})^2 + \cdots + p^2 - (p^K_{0})^2
\end{equation*}
in $p$;
the result is
$$
  p
  =
  \frac1K
  \sum_{k=1}^K
  \left(
    p^k_1 + \frac12 (p^k_0)^2 - \frac12 (p^k_1)^2
  \right).
$$
This expression is more natural than it looks:
see \cite{\OCMVII},
the discussion after~(11);
notice that it reduces to arithmetic mean when $p_0=p_1$.
\fi  

The argument above (``conditioned'' on the proper training set)
is also applicable to IVAP, in which case we need to set $K:=1$;
the probability predictor obtained from an IVAP
by replacing $(p_0,p_1)$ with $p:=p_1/(1-p_0+p_1)$
will be referred to as the \emph{log-minimax IVAP}.
(And CVAP is log-minimax by definition.)

\section{Comparison with other calibration methods}
\label{sec:comparison}

The two alternative calibration methods that we consider in this paper
are Platt's \cite{Platt:2000} and isotonic regression \cite{Zadrozny/Elkan:2001}.

\subsection{Platt's method}

Platt's \cite{Platt:2000} method uses sigmoids\ifnotCONF
  \begin{equation*} 
    g(s)
    :=
    \frac{1}{1+\exp(As+B)},
  \end{equation*}
  where $A<0$ and $B$ are parameters,\fi\
to calibrate the scores.
\ifnotCONF
  Platt discusses two approaches:
  \begin{itemize}
  \item
    run the scoring algorithm and fit the parameters $A$ and $B$ on the full training set,
  \item
    or run the scoring algorithm on a subset (called the proper training set in this paper)
    and fit $A$ and $B$ on the rest (the calibration set).
  \end{itemize}
  Platt recommends the second approach,
  especially that he is interested in SVM,
  and for SVM
  the scores for the training set tend to cluster around $\pm 1$.
  (In fact, this is also true for the calibration scores,
  as discussed below.)

  Platt's recommended method of fitting $A$ and $B$ is
  \begin{equation}\label{eq:min}
    -\sum_{i=1}^k
    \left(
      t_i \log p_i
      +
      (1-t_i) \log (1-p_i)
    \right)
    \to
    \min,
  \end{equation}
  where, in the simplest case, $t_i:=y_i$ are the labels of the calibration observations
  (so that \eqref{eq:min} minimizes the log loss on the calibration set).
  To obtain even better results,
  Platt recommends regularization:
  \begin{equation}\label{eq:t+}
    t_i
    =
    t_+
    :=
    \frac{k_++1}{k_++2}
  \end{equation}
  for the calibration observations labelled 1 (if there are $k_+$ of them) and
  \begin{equation}\label{eq:t-}
    t_i
    =
    t_-
    :=
    \frac{1}{k_-+2}
  \end{equation}
  for the calibration observations labelled 0 (if there are $k_-$ of them).
  We can see from \eqref{eq:t+} and \eqref{eq:t-} that the predictions of Platt's predictor
  are always in the range
\fi
\begin{equation}\label{eq:range}
  \left(
    \frac{1}{k_-+2},
    \frac{k_++1}{k_++2}
  \right).  
\end{equation}
\ifnotCONF
  Let us check that the predictions output by the log-minimax IVAP
  are in the same range as those for Platt's method (except that the end-points are now allowed):
  \begin{lemma}\label{lem:IVAP}
    In the case of IVAP,
    $p_1\ge1/(k_-+1)$ and $p_0\le1-1/(k_++1)$,
    where $k_-$ and $k_+$ are the numbers of positive and negative observations in the calibration set,
    respectively.
    In the case of log-minimax IVAP,
    $p\in[1/(k_-+2),1-1/(k_++2)]$
    (i.e., $p$ is in the closure of \eqref{eq:range}).
    In the case of CVAP,
    $p\in[1/(k+2),1-1/(k+2)]$,
    where $k$ is the size of the largest fold.
  \end{lemma}
  \begin{proof}
    The statement about IVAP is obvious,
    and we will only check that it implies the two other statements.
    For concreteness, we will consider the lower bounds.
    The lower bound $1/(k_-+2)$ for log-minimax IVAP can be deduced from $p_1\ge1/(k_-+1)$
    using the isotonicity of $t/(c+t)$ in $t>0$ for $c>0$:
    \begin{equation*}
      \frac{p_1}{(1-p_0)+p_1}
      \ge
      \frac{1/(k_-+1)}{(1-p_0)+1/(k_-+1)}  
      \ge
      \frac{1/(k_-+1)}{1+1/(k_-+1)}
      =
      \frac{1}{k_-+2}.
    \end{equation*}
    In the same way the lower bound $1/(k+2)$ for CVAP follows from $\GM(p_1)\ge1/(k+1)$:
    \begin{equation*}  
      \frac{\GM(p_1)}{\GM(1-p_0)+\GM(p_1)}
      \ge
      \frac{1/(k+1)}{\GM(1-p_0)+1/(k+1)} 
      \ge
      \frac{1/(k+1)}{1+1/(k+1)}
      =
      \frac{1}{k+2}.
      \tag*{$\qed$}
    \end{equation*}
    \renewcommand{\qedsymbol}{}
  \end{proof}

  It is clear that the end-points of the interval~\eqref{eq:range}
  can be approached arbitrarily closely in the case of Platt's predictor
  and attained in the case of IVAPs.

  The main disadvantage of Platt's method is that the optimal calibration curve $g$
  is quite often far from being a sigmoid;
  and if the training set is very big,
  we will suffer, since in this case we can learn the best shape of the calibrator $g$.
  This is particularly serious in asymptotics as the amount of data tends to infinity.


  Zhang \cite{Zhang:2004} (Section~3.3) observes that in the case of SVM
  and universal \cite{Steinwart:2001} kernels the scores tend to cluster around $\pm1$
  at ``non-trivial'' objects, i.e., objects that are labelled 1
  with non-trivial (not close to 0 or 1) probability.
  This means that any sigmoid will be a poor calibrator unless the prediction problem is very easy.
  Formally, we have the following statement
  (a trivial corollary of known results),
  which uses the notation $\eta(x)$ for the conditional probability
  that the label of an object $x\in\mathbf{X}$ is~1
  and assumes that the labels take values in $\{-1,1\}$, $y_i\in\{-1,1\}$
  (rather than $y_i\in\{0,1\}$, as in the rest of this paper).
  \begin{proposition}
    Suppose that the probability of each of the events
    $\eta(X)=0$, $\eta(X)=1/2$, and $\eta(X)=1$ is 0.
    Let $f_m$ be the SVM for a training set of size $m$,
    i.e., the solution to the optimization problem
    \begin{equation}\label{eq:SVM}
      C_m \left\|f\right\|^2_H
      +
      \sum_{i=1}^m
      \phi(f(x_i)y_i)
      \to
      \min,
    \end{equation}
    where $\phi(v):=(1-v)^+$ and $H$ is a universal RKHS (\cite{Steinwart/Christmann:2008}, Definition~4.52).
    As $m\to\infty$,
    \begin{equation*} 
      f_m(X)
      \to
      f(X)
      :=
      \begin{cases}
        -1 & \text{if $\eta(X)\in[0,1/2]$}\\
        1 & \text{if $\eta(X)\in(1/2,1]$}
      \end{cases}
    \end{equation*}
    in probability
    provided $C_m\to\infty$ and $C_m=o(m)$.
  \end{proposition}
  \begin{proof}
    This follows immediately from Theorem~4.4 in \cite{Zhang:2004}
    for a natural class of universal kernels related to neural networks.
    In general, see the proof of Theorem~8.1 in \cite{Steinwart/Christmann:2008}.
  \end{proof}

  The intuition behind the SVM decision values clustering around $\pm1$ is very simple.
  SVM solves the optimization problem \eqref{eq:SVM};
  asymptotically as $m\to\infty$ and under natural assumptions
  (such as $C_m\to\infty$ and $C_m=o(m)$),
  this solves
  $$
    \Expect
    \phi(f(X)Y)
    \to
    \min.
  $$
  We can optimize separately for different values of $\eta(x)$.
  Given $\eta(x)=\eta^*$, we have the optimization problem
  $$
    \eta^*\phi(f) + (1-\eta^*)\phi(-f) \to \min,
  $$
  whose solutions are
  \begin{equation*} 
    f(x)
    \in
    \begin{cases}
      (-\infty,-1] & \text{if $\eta(x)=0$}\\
      \{-1\} & \text{if $\eta(x)\in(0,1/2)$}\\
      [-1,1] & \text{if $\eta(x)=1/2$}\\
      \{1\} & \text{if $\eta(x)\in(1/2,1)$}\\
      [1,\infty) & \text{if $\eta(x)=1$}.
    \end{cases}
  \end{equation*}

  Assuming that the probability of each of the events
  $\eta(X)=0$, $\eta(X)=1/2$, and $\eta(X)=1$ is 0,
  it is easy to check that asymptotically the best achievable excess log loss of a sigmoid
  over the Bayes algorithm is
  \begin{equation}  
  \label{eq:excess}
    \Expect
    \Bigl(
      \KL
      \left(
        \eta
        \:\middle||\:
        \Expect(\eta\mid\eta>1/2)
      \right)
      \III_{\eta>1/2} 
      +
      \KL
      \left(
        \eta
        \:\middle||\:
        \Expect(\eta\mid\eta<1/2)
      \right)
      \III_{\eta<1/2}
    \Bigr),
  \end{equation}
  where $\KL$ is Kullback--Leibler divergence defined in terms of base 2 logarithm $\log_2$,
  and the conditional expectation $\Expect(\eta\mid E)$
  is defined to be $\Expect(\eta\III_E)/\Prob(E)$.

  On the other hand, there are no apparent obstacles
  to it approaching 0 in the case of isotonic regression,
  considered in the next subsection.

  For illustration, suppose $\eta:=\eta(X)$ is distributed uniformly in $[0,1]$.
  It is easy to see that
  \begin{align*}
    \Expect(\eta\mid\eta>1/2) &= 3/4\\
    \Expect(\eta\mid\eta<1/2) &= 1/4;
  \end{align*}
  therefore, the excess loss \eqref{eq:excess} is
  \begin{multline*}
    \Expect
    \Bigl(
      \KL
      \left(
        \eta
        \:\middle||\:
        3/4
      \right)
      \III_{\eta>1/2}
      +
      \KL
      \left(
        \eta
        \:\middle||\:
        1/4
      \right)
      \III_{\eta<1/2}
    \Bigr) 
    =
    \Expect
    \Bigl(
      \eta\log_2\eta
      +
      (1-\eta)\log_2(1-\eta)
    \Bigr)\\   
    +
    2
    \Expect
    \left(
      \eta 1_{\eta>1/2} \log_2\frac43
      +
      \eta 1_{\eta<1/2} \log_2 4
    \right) 
    \approx
    -0.7213 + 0.8113
    =
    0.09.
  \end{multline*}
  We can see that the Bayes log loss is $72.13\%$,
  whereas the best loss achievable by a sigmoid is $81.13\%$,
  9 percentage points worse.
\fi  

\subsection{Isotonic regression}

There are two standard uses of isotonic regression:
we can train the scoring algorithm using what we call a proper training set,
and then use the scores of the observations in a disjoint calibration (also called validation) set
for calibrating the scores of test objects (as in \cite{Caruana/NM:2006});
alternatively, we can train the scoring algorithm on the full training set
and also use the full training set for calibration
(it appears that this was done in \cite{Zadrozny/Elkan:2001}).
In both cases, however, we can expect to get an infinite log loss when the test set becomes large enough.
\ifnotCONF
  Indeed, suppose that we have fixed proper training and calibration sets
  (not necessarily disjoint, so that both cases mentioned above are covered)
  such that the score $s(X)$ of a random object $X$
  is below the smallest score of the calibration objects with a positive probability;
  suppose also that the distribution of the label of a random observation
  is concentrated at 0 with probability zero.
  Under these realistic assumptions
  the probability that the average log loss on the test set is $\infty$
  can be made arbitrarily close to one
  by making the size of the test set large enough:
  indeed, with a high probability there will be an observation $(x,y)$ in the test set
  such that the score $s(x)$ is below the smallest score of the calibration objects but $y=1$;
  the log loss on such an observation will be infinite.
\fi

The presence of regularization is an advantage of Platt's method:
e.g., it never suffers an infinite loss when using the log loss function.
There is no standard method of regularization for isotonic regression,
and we do not apply one\ifnotCONF\footnote{One of the reviewers
  of the conference version of this paper
  proposed complementing the calibration set used in isotonic regression
  by two dummy observations:
  one with score $+\infty$ and labelled by $0$
  and the other with score $-\infty$ and labelled by $1$.}\fi.

\section{Empirical studies}
\label{sec:experiments}

The main loss function (cf., e.g., \cite{Vovk:2015Yuri})
that we use in our empirical studies is the \emph{log loss}
\begin{equation}\label{eq:log-loss}
  \lambda_{\log}(p,y)
  :=
  \begin{cases}
    -\log p & \text{if $y=1$}\\
    -\log(1-p) & \text{if $y=0$},
  \end{cases}
\end{equation}
where $\log$ is binary logarithm,
$p\in[0,1]$ is a probability prediction, and $y\in\{0,1\}$ is the true label.
Another popular loss function is the \emph{Brier loss}
\begin{equation}\label{eq:Brier-loss}
  \lambda_{\Br}(p,y)
  :=
  4(y-p)^2.
\end{equation}
We choose the coefficient 4 in front of $(y-p)^2$ in \eqref{eq:Brier-loss}
and the base 2 of the logarithm in~\eqref{eq:log-loss} in order for the minimax no-information predictor
that always predicts $p:=1/2$ to suffer loss~1.
An advantage of the Brier loss function
is that it still makes it possible to compare the quality of prediction
in cases when prediction algorithms (such as isotonic regression) give a categorical but wrong prediction
(and so are simply regarded as infinitely bad when using log loss).

The loss of a probability predictor on a test set will be measured
by the arithmetic average of the losses it suffers on the test set,
namely, by the \emph{mean log loss} (MLL) and the \emph{mean Brier loss} (MBL)
\begin{equation}\label{eq:losses}  
  \MLL
  :=
  \frac1n
  \sum_{i=1}^n
  \lambda_{\log}(p_i,y_i), 
  \quad
  \MBL
  :=
  \frac1n
  \sum_{i=1}^n
  \lambda_{\Br}(p_i,y_i), 
\end{equation}
where $y_i$ are the test labels and $p_i$ are the probability predictions for them.
We will not be checking directly whether various calibration methods
produce well-calibrated predictions,
since it is well known that lack of calibration increases the loss
as measured by loss functions such as log loss and Brier loss
(see, e.g., \cite{Murphy:1973} for the most standard decomposition
of the latter into the sum of the calibration error and refinement error).

In this section we compare log-minimax IVAPs
(i.e., IVAPs whose outputs are replaced by probability predictions,
as explained in Section~\ref{sec:merging})
and CVAPs with Platt's method \cite{Platt:2000}
and the standard method \cite{Zadrozny/Elkan:2001} based on isotonic regression;
the latter two will be referred to as ``Platt'' and ``Isotonic''
in our tables and figures.
(Even though for both IVAPs and CVAPs we use the log-minimax procedure
for merging multiprobability predictions,
the Brier-minimax procedure leads to virtually identical empirical results.)
We use the same underlying algorithms as in \cite{\OCMVII},
namely J48 decision trees (abbreviated to ``J48''),
J48 decision trees with bagging (``J48 bagging''),
logistic regression (sometimes abbreviated to ``logistic''),
naive Bayes, neural networks, and support vector machines (SVM),
as implemented in Weka \cite{Weka:2011} (University of Waikato, New Zealand).
The underlying algorithms (except for SVM) produce scores in the interval $[0,1]$,
which can be used directly as probability predictions
(referred to as ``Underlying'' in our tables and figures)
or can be calibrated using the methods of \cite{Platt:2000,Zadrozny/Elkan:2001}
or the methods proposed in this paper (``IVAP'' or ``CVAP'' in the tables and figures).

We start our empirical studies with
the \texttt{adult} data set
available from the UCI repository \cite{UCI:data}
(this is the main data set used in \cite{Platt:2000}
and one of the data sets used in \cite{Zadrozny/Elkan:2001});
however,
as we will see later,
the picture that we observe is typical for other data sets as well.
We use the original split of the data set into a training set of $N_{\rm train}=32,561$ observations
and a test set of $N_{\rm test}=16,281$ observations.
The results of applying the four calibration methods
(plus the vacuous one, corresponding to just using the underlying algorithm)
to the six underlying algorithms
for this data set are shown in
\ifPortrait Figures~\ref{fig:log-adult} and~\ref{fig:Brier-adult}\fi.
\ifPortrait
  Figures~\ref{fig:log-adult} reports results for the log loss
  (namely, $\MLL$, as defined in \eqref{eq:losses})
  and Figure~\ref{fig:Brier-adult} for the Brier loss
  (namely, $\MBL$).
\fi%
The underlying algorithms are given in the titles of the plots
and the calibration methods are represented by different line styles,
as explained in the legends.
The marks on the horizontal axis are the ratios of the size of the proper training set to the size of the calibration set
(except for the label \texttt{all}, which will be explained later);
in the case of CVAPs, the number $K$ of folds can be expressed as the sum of the two numbers forming the ratio
(therefore, column~4:1 corresponds to the standard choice of 5 folds in the method of cross-validation).
Missing curves or points on curves mean that the corresponding values either are too big
and would squeeze unacceptably the interesting parts of the plot if shown
or are infinite (such as many results for isotonic regression and neural networks under log loss).
In the case of CVAPs, the training set is split into $K$ equal
(or as close to being equal as possible) contiguous folds:
the first $\lceil N_{\rm train}/K\rceil$ training observations are included in the first fold,
the next $\lceil N_{\rm train}/K\rceil$ (or $\lfloor N_{\rm train}/K\rfloor$)
in the second fold, etc.\
(first $\lceil\cdot\rceil$ and then $\lfloor\cdot\rfloor$ is used
unless $N_{\rm train}$ is divisible by $K$).
In the case of the other calibration methods,
we used the first $\lceil\frac{K-1}{K}N_{\rm train}\rceil$ training observation
as the proper training set (used for training the scoring algorithm)
and the rest of the training observations are used as the calibration set.

\ifPortrait
  \begin{figure*} 
    \begin{center}
      \includegraphics[trim = 0mm 10mm 0mm 12mm, clip, width=\textwidth]{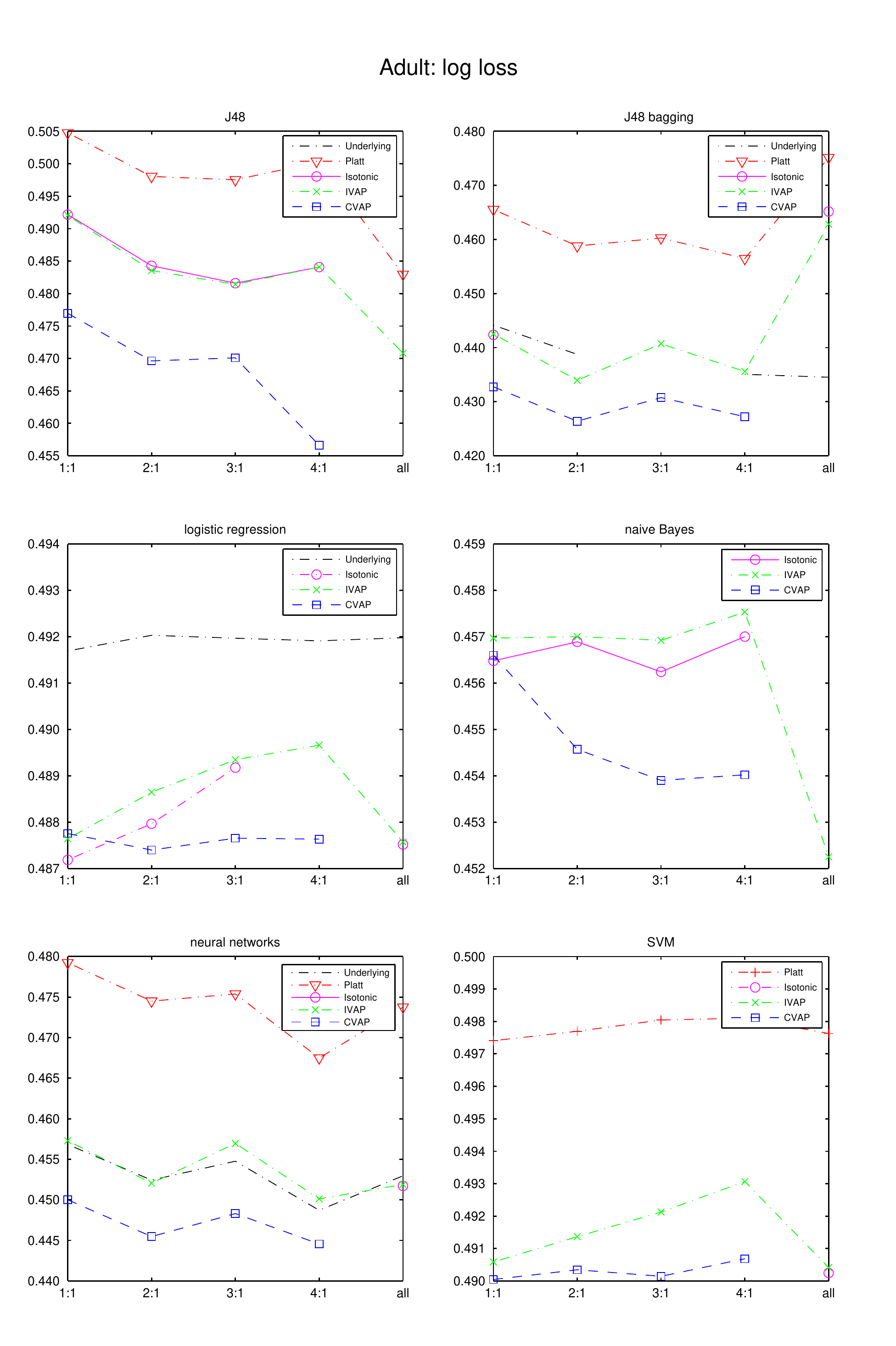}
    \end{center}
    \caption{The log losses of the four calibration methods
      applied to the six prediction algorithms on the \texttt{adult} data set.}
    \label{fig:log-adult}
  \end{figure*}

  \begin{figure*} 
    \begin{center}
      \includegraphics[trim = 0mm 10mm 0mm 12mm, clip, width=\textwidth]{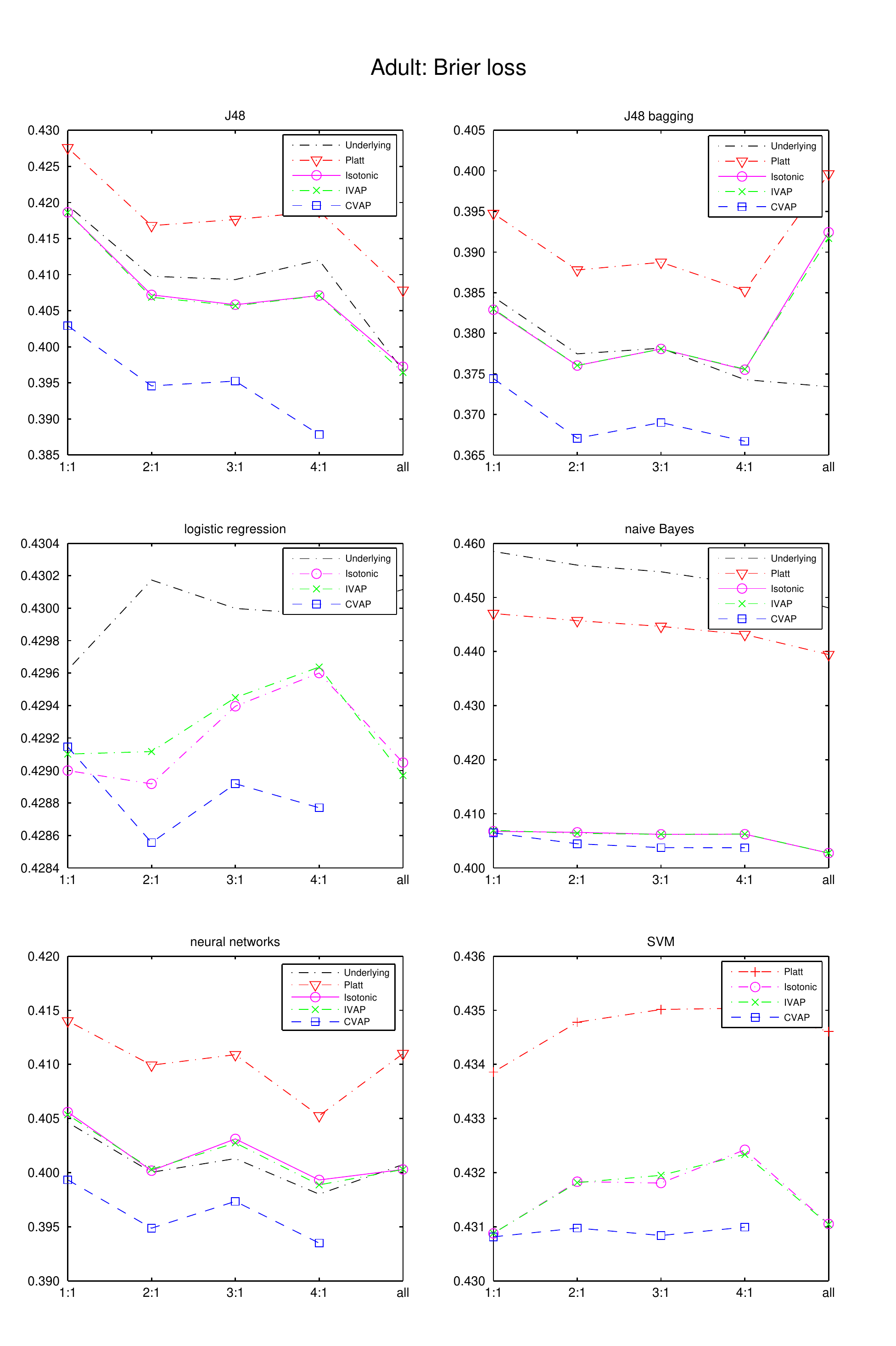}
    \end{center}
    \caption{The analogue of Figure~\ref{fig:log-adult} for Brier loss.}
    \label{fig:Brier-adult}
  \end{figure*}
\fi

In the case of log loss, isotonic regression often suffers infinite losses,
which is indicated by the absence of the round marker for isotonic regression;
e.g., only one of the log losses for SVM is finite.
We are not trying to use ad hoc solutions, such as clipping predictions
to the interval $[\epsilon,1-\epsilon]$ for a small $\epsilon>0$,
since we are also using the bounded Brier loss function.
The CVAP lines tend to be at the bottom in all plots;
experiments with other data sets also confirm this.

The column \texttt{all} in the plots of
\ifPortrait Figures~\ref{fig:log-adult} and~\ref{fig:Brier-adult}\fi\
refers to using the full training set as both the proper training set and calibration set.
(In our official definition of IVAP we require that the last two sets be disjoint,
but in this section we continue to refer to IVAPs modified in this way simply as IVAPs;
in \cite{\OCMVII}, such prediction algorithms were referred to as SVAPs,
simplified Venn--Abers predictors.)
Using the full training set as both the proper training set and calibration set
might appear naive
(and is never used in the extensive empirical study \cite{Caruana/NM:2006}),
but it often leads to good empirical results on larger data sets.
However, it can also lead to very poor results,
as in the case of ``J48 bagging'' (for IVAP, Platt, and Isotonic),
the underlying algorithm that achieves the best performance in
\ifPortrait Figures~\ref{fig:log-adult} and~\ref{fig:Brier-adult}\fi.

A natural question is whether CVAPs perform better than the alternative calibration methods in
\ifPortrait Figures~\ref{fig:log-adult} and~\ref{fig:Brier-adult}\fi\
(and our other experiments) because of applying cross-over (in moving from IVAP to CVAP)
or because of the extra regularization used in IVAPs.
The first reason is undoubtedly important for both loss functions
and the second for the log loss function.
The second reason plays a smaller role for Brier loss for relatively large data sets
(in \ifPortrait Figure~\ref{fig:Brier-adult}\fi\
the curves for \texttt{Isotonic} and \texttt{IVAP} are very close to each other),
but IVAPs are consistently better for smaller data sets even when using Brier loss.
In Tables~\ref{tab:log} and~\ref{tab:Brier}
we apply the four calibration methods and six underlying algorithms
to a much smaller training set,
namely to the first $5,000$ observations of the \texttt{adult} data set as the new training set,
following \cite{Caruana/NM:2006};
the first $4,000$ training observations are used as the proper training set,
the following $1,000$ training observations as the calibration set,
and all other observations (the remaining training and all test observations)
are used as the new test set.
The results are shown in Tables~\ref{tab:log} for log loss
and~\ref{tab:Brier} for Brier loss.
They are consistently better for IVAP than for IR (isotonic regression).
Results for nine very small data sets are given in Tables~1 and~2 of \cite{\OCMVII},
where the results for IVAP
(with the full training set used as both proper training and calibration sets,
labelled ``SVA'' in the tables in \cite{\OCMVII})
are consistently (in 52 cases out of the 54 using Brier loss) better,
usually significantly better, than for isotonic regression
(referred to as DIR in the tables in \cite{\OCMVII}).

\begin{table}
\caption{The log loss for the four calibration methods and six underlying algorithms
  for a small subset of the \texttt{adult} data set}\label{tab:log}
  \begin{center}
    \begin{tabular}{c|ccccc}
      algorithm & Platt & IR & IVAP & CVAP \\
      \hline
      J48 & 0.5226 & $\infty$ & 0.5117 & 0.5102 \\
      J48 bagging & 0.4949 & $\infty$ & 0.4733 & 0.4602 \\
      logistic & 0.5111 & $\infty$ & 0.4981 & 0.4948 \\
      naive Bayes & 0.5534 & $\infty$ & 0.4839 & 0.4747 \\
      neural networks & 0.5175 & $\infty$ & 0.5023 & 0.4805 \\
      SVM & 0.5221 & $\infty$ & 0.5015 & 0.4997
    \end{tabular}
  \end{center}
\end{table}

\begin{table}
\caption{The analogue of Table~\ref{tab:log} for the Brier loss}\label{tab:Brier}
  \begin{center}
    \begin{tabular}{c|ccccc}
      algorithm & Platt & IR & IVAP & CVAP \\
      \hline
      J48 & 0.4463 & 0.4378 & 0.4370 & 0.4368 \\
      J48 bagging & 0.4225 & 0.4153 & 0.4123 & 0.3990 \\
      logistic & 0.4470 & 0.4417 & 0.4377 & 0.4342 \\
      naive Bayes & 0.4670 & 0.4329 & 0.4311 & 0.4227 \\
      neural networks & 0.4525 & 0.4574 & 0.4440 & 0.4234 \\
      SVM & 0.4550 & 0.4450 & 0.4408 & 0.4375
    \end{tabular}
  \end{center}
\end{table}

The following information might help the reader in reproducing our results
(in addition to our code being posted on arXiv together with this paper).
For each of the standard prediction algorithms within Weka that we use,
we optimise the parameters by minimising the Brier loss on the calibration set,
apart from the column labelled \texttt{all}.
(We cannot use the log loss since it is often infinite in the case of isotonic regression.)
We then use the trained algorithm to generate the scores for the calibration and test sets,
which allows us to compute probability predictions
using Platt's method, isotonic regression, IVAP, and CVAP.
All the scores apart from SVM are already in the $[0,1]$ range and can be used as probability predictions.
%
Most of the parameters are set to their default values,
and the only parameters that are optimised are \texttt{C} (pruning confidence) for J48 and J48 bagging,
\texttt{R} (ridge) for logistic regression,
\texttt{L} (learning rate) and \texttt{M} (momentum) for neural networks (\texttt{MultilayerPerceptron}),
and \texttt{C} (complexity constant) for SVM (\texttt{SMO}, with the linear kernel);
naive Bayes does not involve any parameters.
Notice that none of these parameters are ``hyperparameters'',
in that they do not control the flexibility of the fitted prediction rule directly;
this allows us to optimize the parameters on the training set for the \texttt{all} column.
In the case of CVAPs, we optimise the parameters by minimising the cumulative Brier loss over all folds
(so that the same parameters are used for all folds).
To apply Platt's method to calibrate the scores generated by the underlying algorithms
we use logistic regression, namely the function \texttt{mnrfit} within MATLAB's Statistics toolbox.
For isotonic regression calibration we use the implementation of the PAVA in the R package \texttt{fdrtool}
(namely, the function \texttt{monoreg}).
Missing values are handled using the Weka filter \texttt{ReplaceMissingValues},
which replaces all missing values for nominal and numeric attributes with the modes and means from the training set.

\subsection*{Additional experimental results}

\ifPortrait Figures~\ref{fig:log-forest} and~\ref{fig:Brier-forest} show \fi
our results for the \texttt{covertype} data set
(available from the UCI repository \cite{UCI:data} and also known as \texttt{forest}).
In converting this multiclass classification problem to binary we follow \cite{Caruana/NM:2006}:
treat the largest class as 1 and the rest as 0,
and only consider a random and randomly permuted subset consisting of $30,000$ observations;
the first $5000$ of those observations are used as the training set and the remaining $25,000$ as the test set.
The CVAP results are still at the bottom of the plots and very stable;
and the values at the \texttt{all} column are still particularly unstable.

\ifPortrait
  \begin{figure*} 
    \begin{center}
      \includegraphics[trim = 0mm 10mm 0mm 12mm, clip, width=\textwidth]{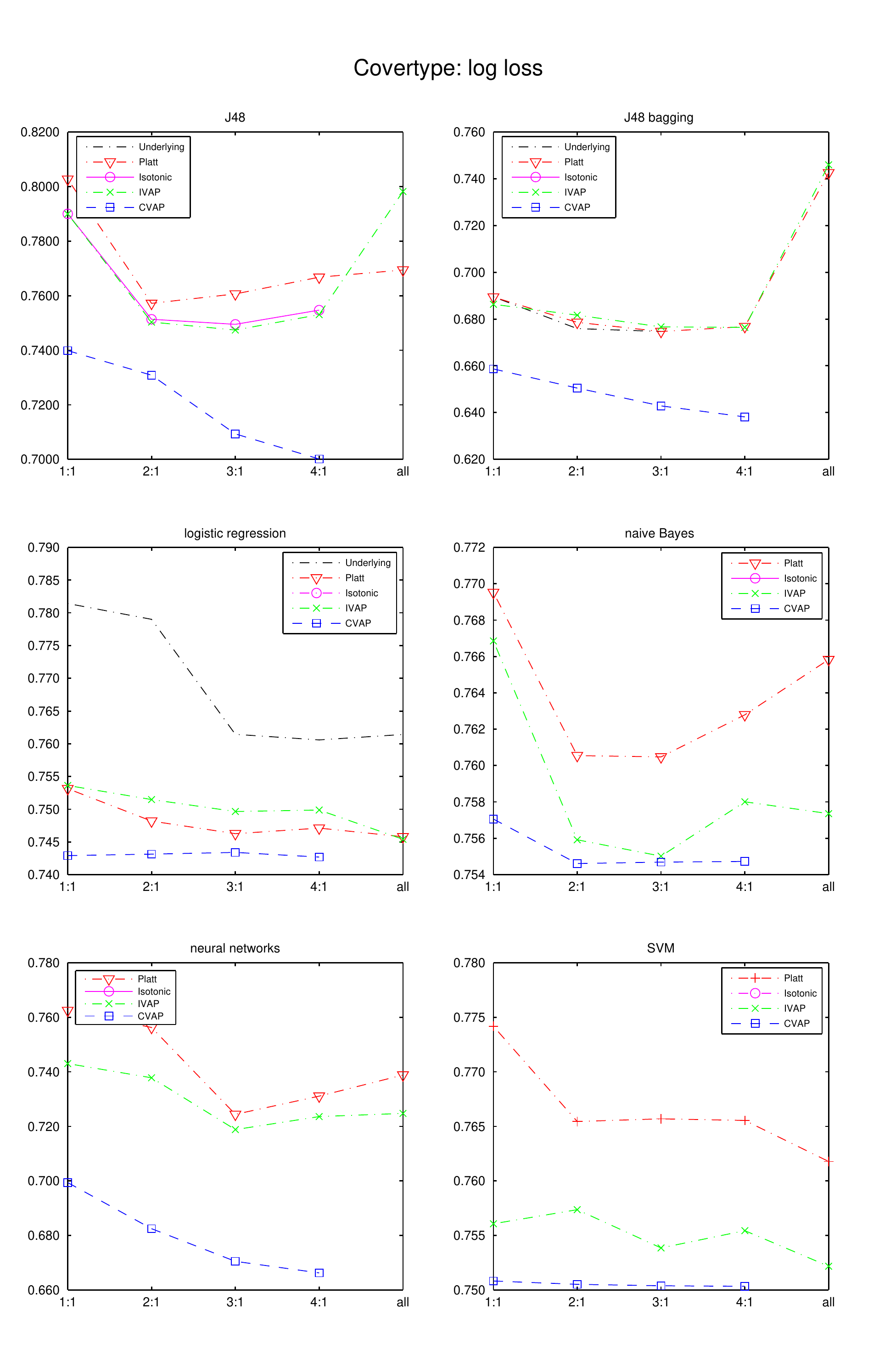}
    \end{center}
    \caption{The analogue of Figure~\ref{fig:log-adult} for the \texttt{covertype} data set.}
    \label{fig:log-forest}
  \end{figure*}

  \begin{figure*} 
    \begin{center}
      \includegraphics[trim = 0mm 10mm 0mm 12mm, clip, width=\textwidth]{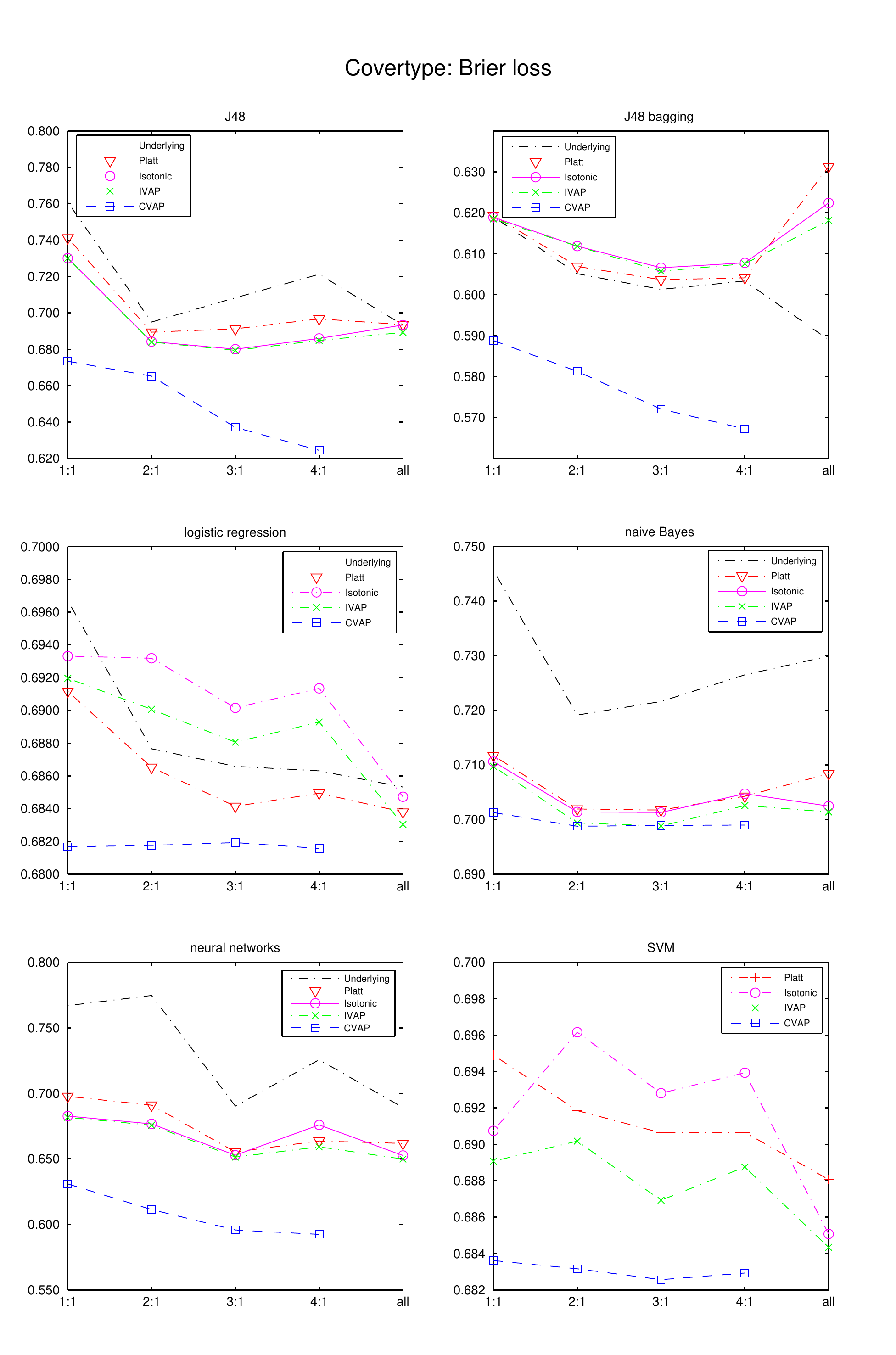}
    \end{center}
    \caption{The analogue of Figure~\ref{fig:Brier-adult} for the \texttt{covertype} data set.}
    \label{fig:Brier-forest}
  \end{figure*}
\fi

Similar results for the \texttt{insurance}, \texttt{Bank Marketing}, \texttt{Spambase},
and \texttt{Statlog German Credit Data} data sets
are shown in
\ifPortrait Figures~\ref{fig:log-insurance}--\ref{fig:Brier-Statlog}. \fi
The data sets are split into training and test sets in proportion 2:1,
without randomization.
%
Since the values for the \texttt{all} column are so unstable,
the reader might prefer to disregard them in the case of IVAP, Platt, and Isotonic.
In
\ifPortrait Figures~\ref{fig:log-insurance}--\ref{fig:Brier-Spambase} \fi
the CVAP results tend to be at the bottom of the plots.
The \texttt{Statlog German Credit Data} data set is much more difficult,
and all results in
\ifPortrait Figures~\ref{fig:log-Statlog}--\ref{fig:Brier-Statlog} \fi
are poor and somewhat mixed;
however, they still demonstrate that CVAPs and IVAPs produce stable results and avoid the occasional bad failures
characteristic of the alternative calibration methods.

\ifPortrait
  \begin{figure*} 
    \begin{center}
      \includegraphics[trim = 0mm 10mm 0mm 12mm, clip, width=\textwidth]{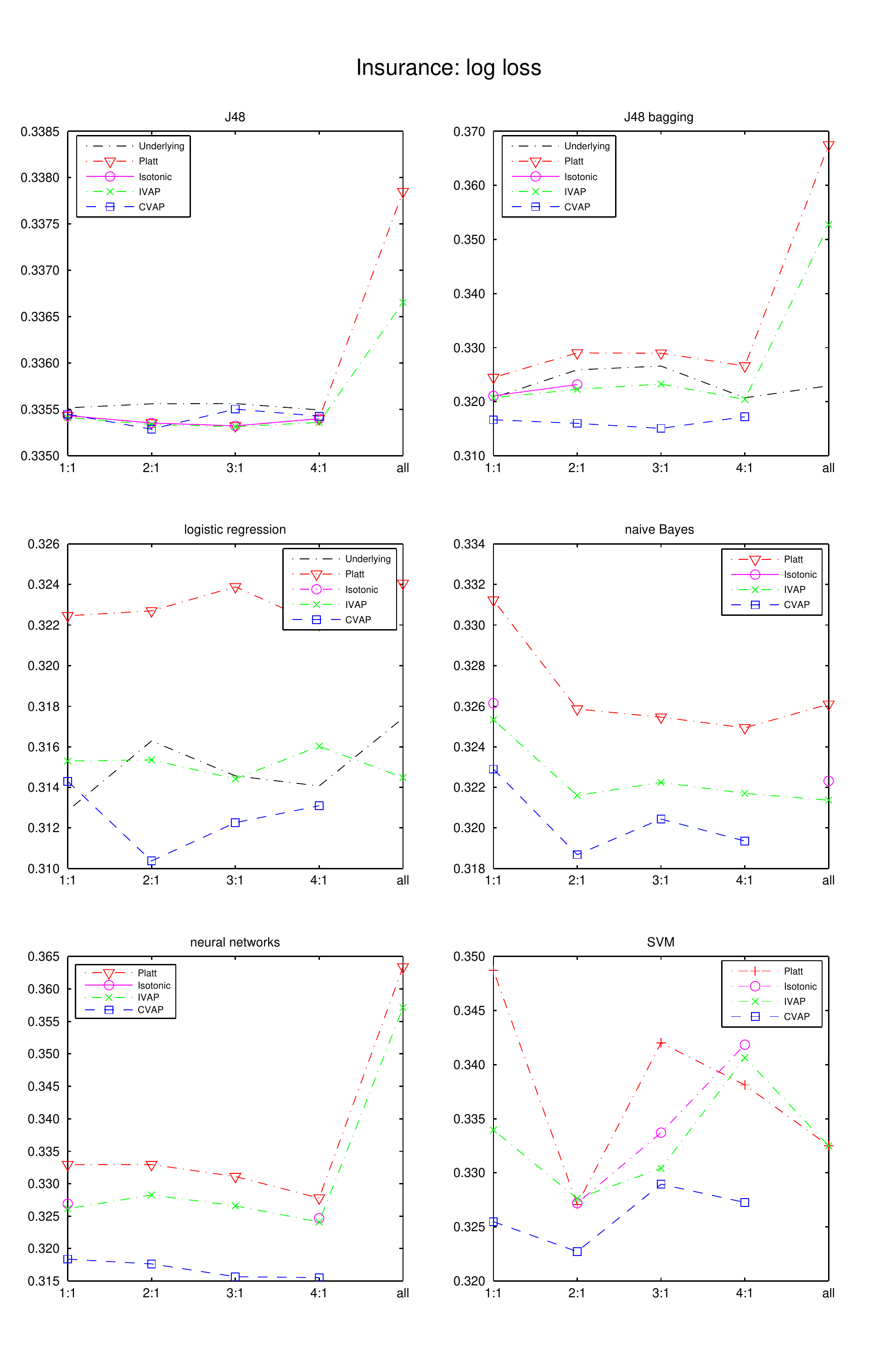}
    \end{center}
    \caption{The analogue of Figure~\ref{fig:log-adult} for the \texttt{insurance} data set.}
    \label{fig:log-insurance}
  \end{figure*}

  \begin{figure*} 
    \begin{center}
      \includegraphics[trim = 0mm 10mm 0mm 12mm, clip, width=\textwidth]{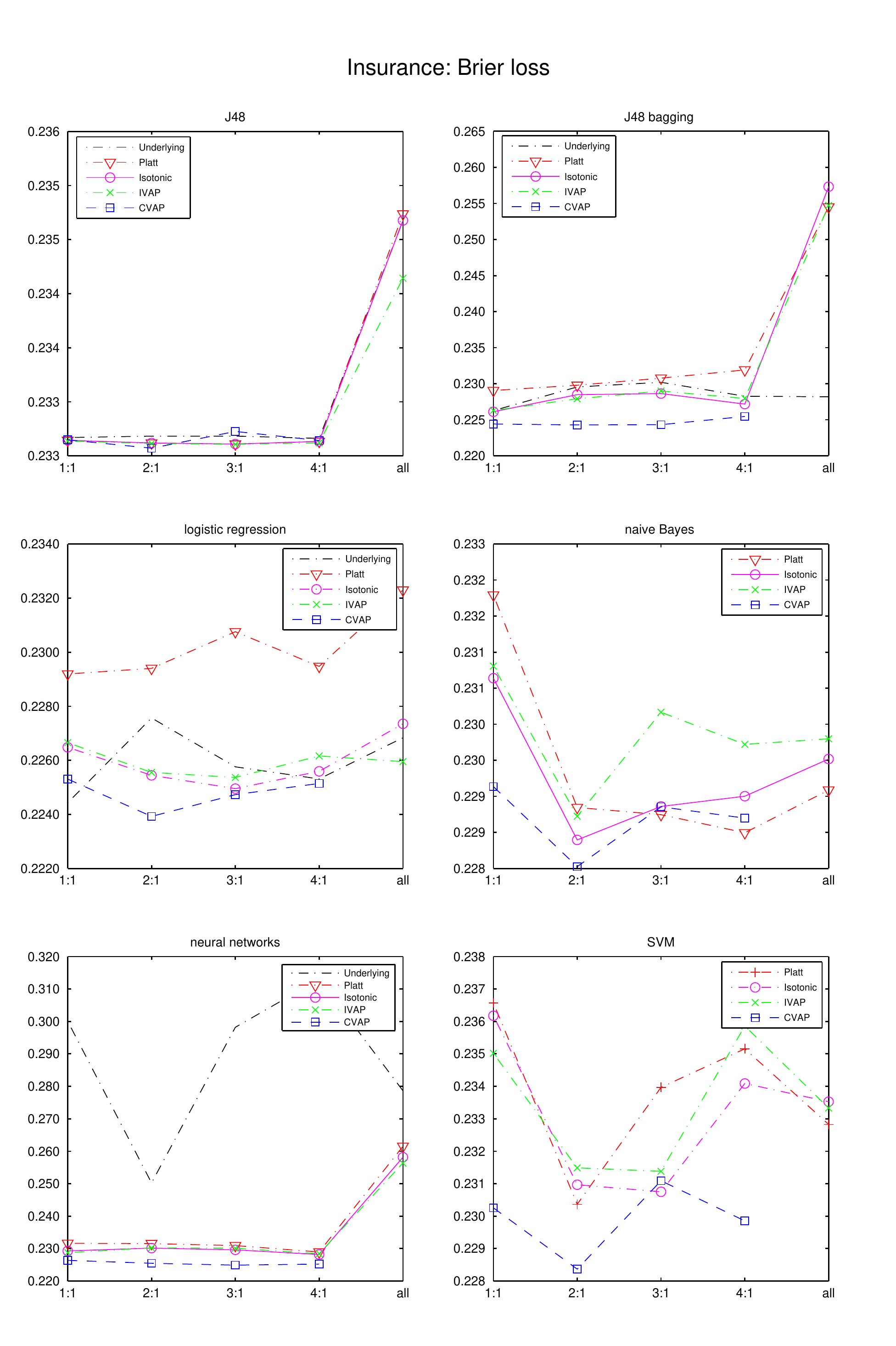}
    \end{center}
    \caption{The analogue of Figure~\ref{fig:Brier-adult} for the \texttt{insurance} data set.}
    \label{fig:Brier-insurance}
  \end{figure*}
\fi

\ifPortrait
  \begin{figure*} 
    \begin{center}
      \includegraphics[trim = 0mm 10mm 0mm 12mm, clip, width=\textwidth]{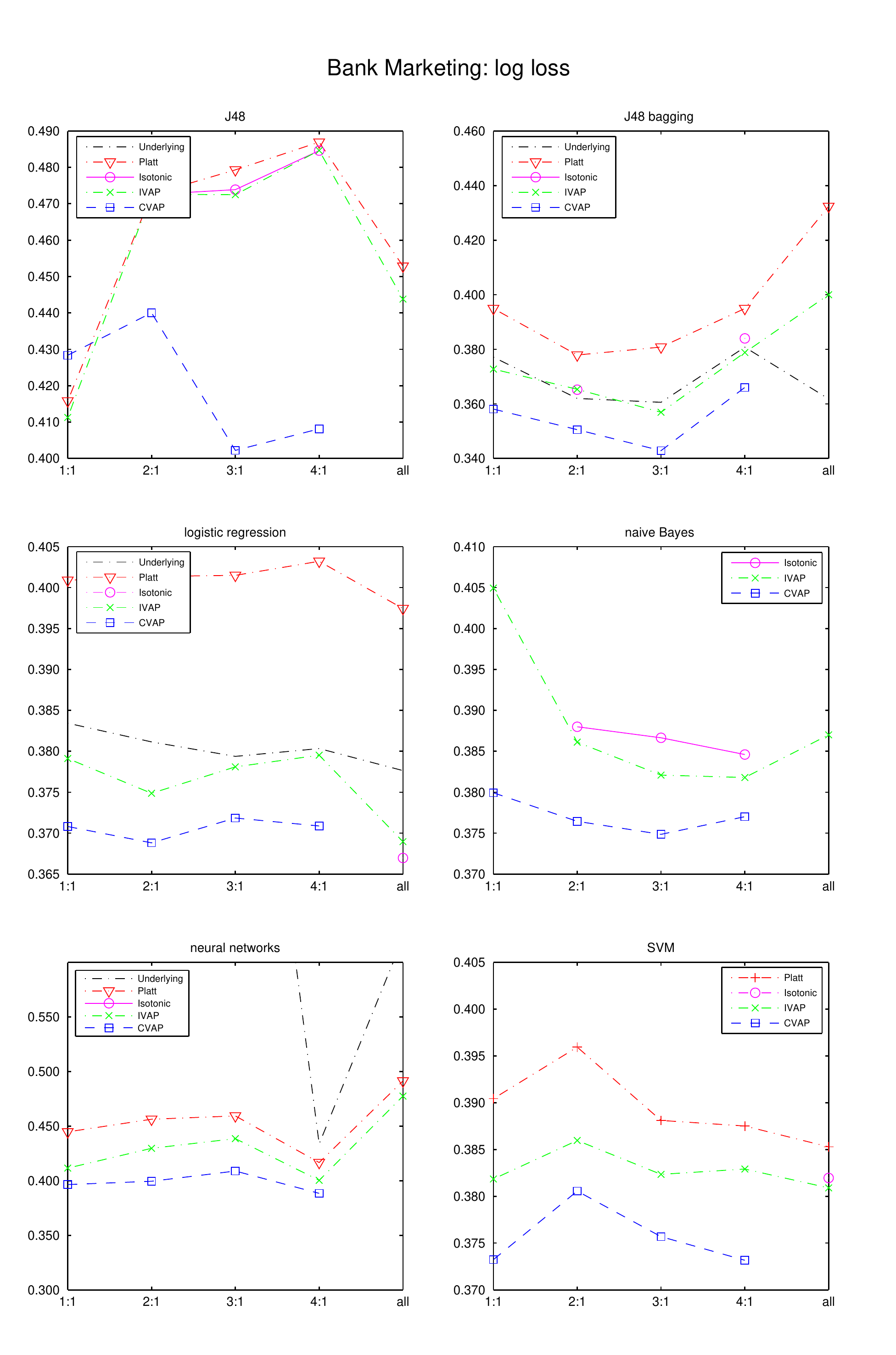}
    \end{center}
    \caption{The analogue of Figure~\ref{fig:log-adult} for the \texttt{Bank Marketing} data set.}
    \label{fig:log-bank}
  \end{figure*}

  \begin{figure*} 
    \begin{center}
      \includegraphics[trim = 0mm 10mm 0mm 12mm, clip, width=\textwidth]{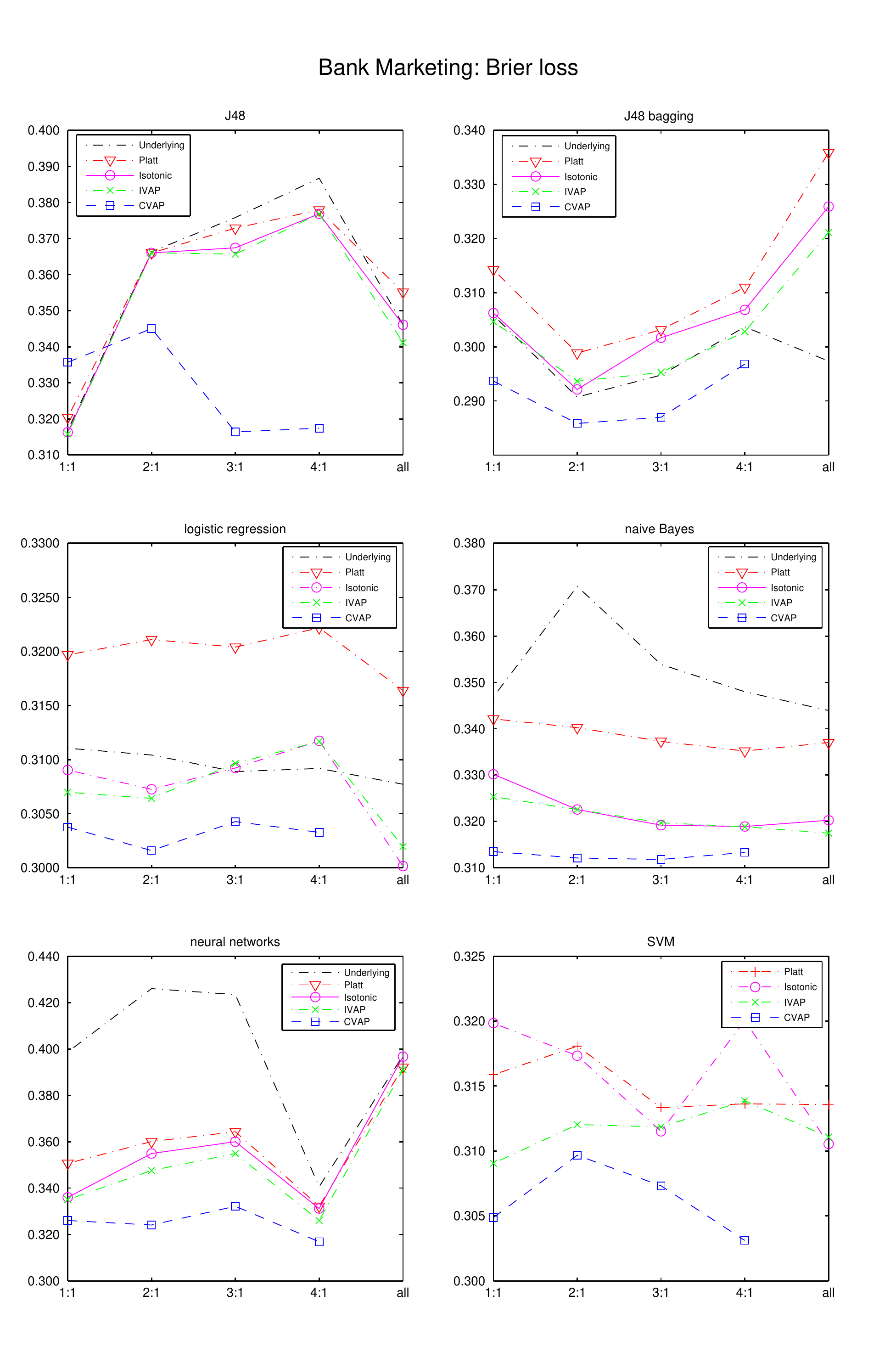}
    \end{center}
    \caption{The analogue of Figure~\ref{fig:Brier-adult} for the \texttt{Bank Marketing} data set.}
    \label{fig:Brier-bank}
  \end{figure*}
\fi

\ifPortrait
  \begin{figure*} 
    \begin{center}
      \includegraphics[trim = 0mm 10mm 0mm 12mm, clip, width=\textwidth]{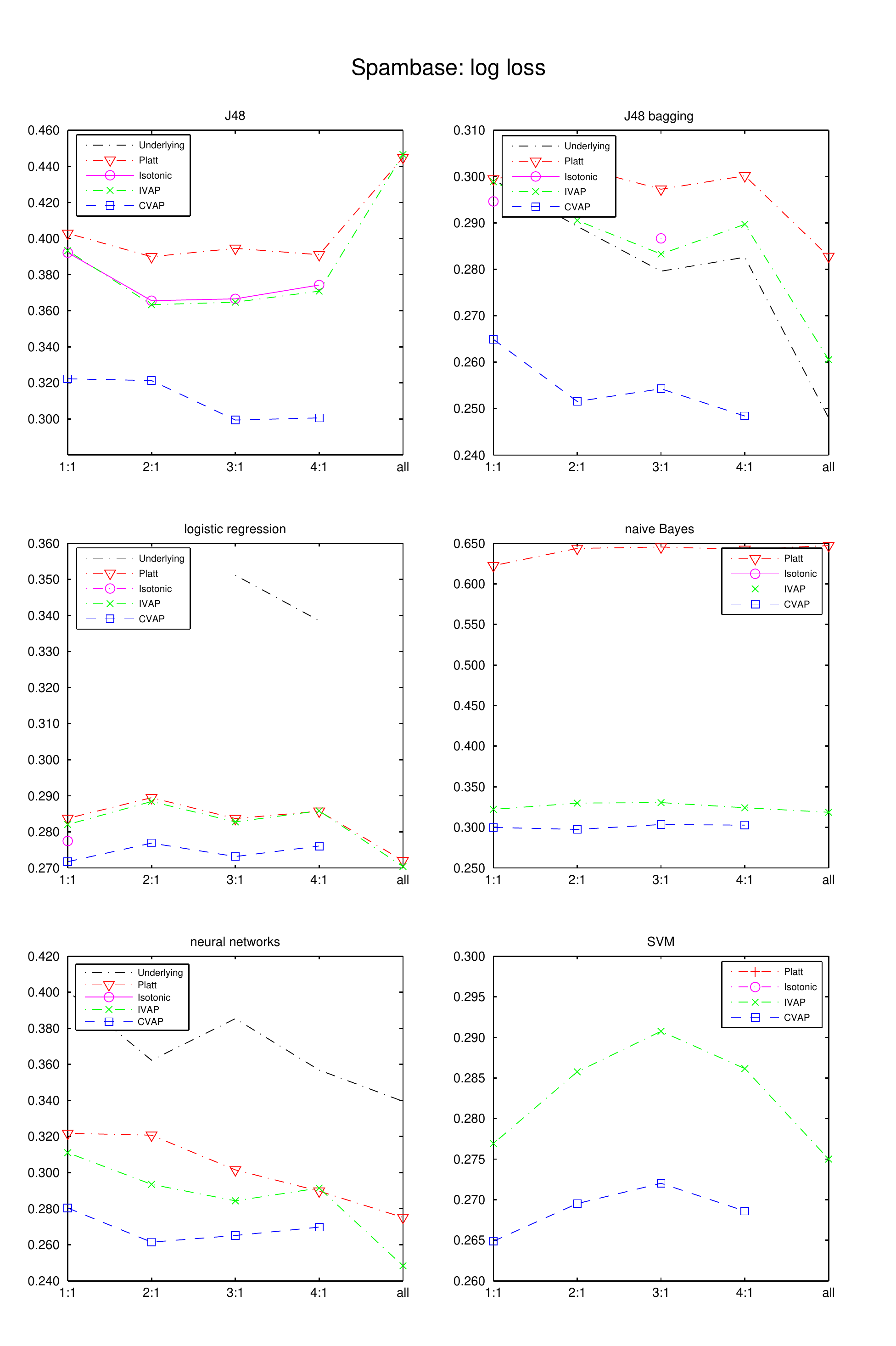}
    \end{center}
    \caption{The analogue of Figure~\ref{fig:log-adult} for the \texttt{Spambase} data set.}
    \label{fig:log-Spambase}
  \end{figure*}

  \begin{figure*} 
    \begin{center}
      \includegraphics[trim = 0mm 10mm 0mm 12mm, clip, width=\textwidth]{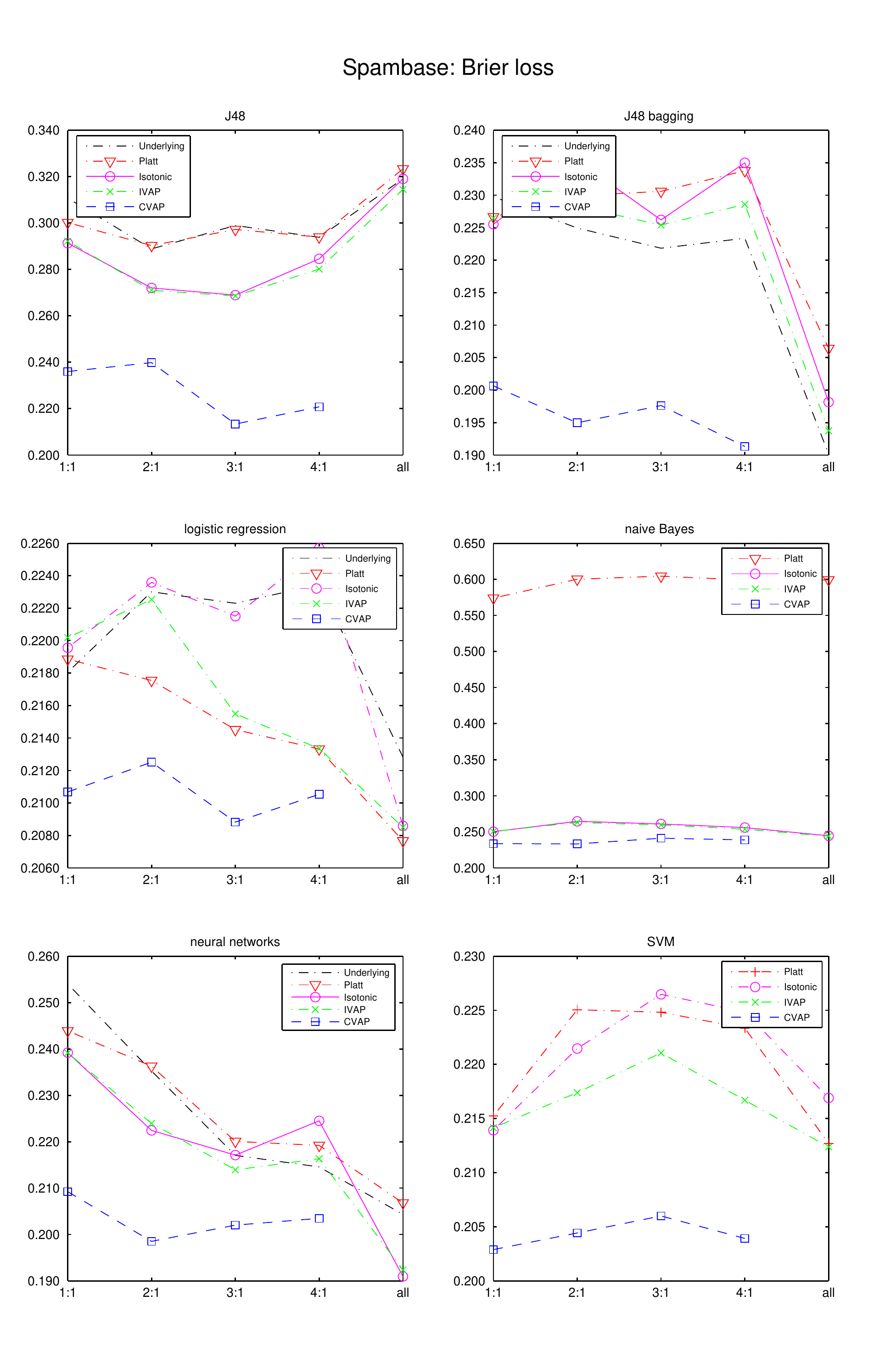}
    \end{center}
    \caption{The analogue of Figure~\ref{fig:Brier-adult} for the \texttt{Spambase} data set.}
    \label{fig:Brier-Spambase}
  \end{figure*}
\fi

\ifPortrait
  \begin{figure*} 
    \begin{center}
      \includegraphics[trim = 0mm 10mm 0mm 12mm, clip, width=\textwidth]{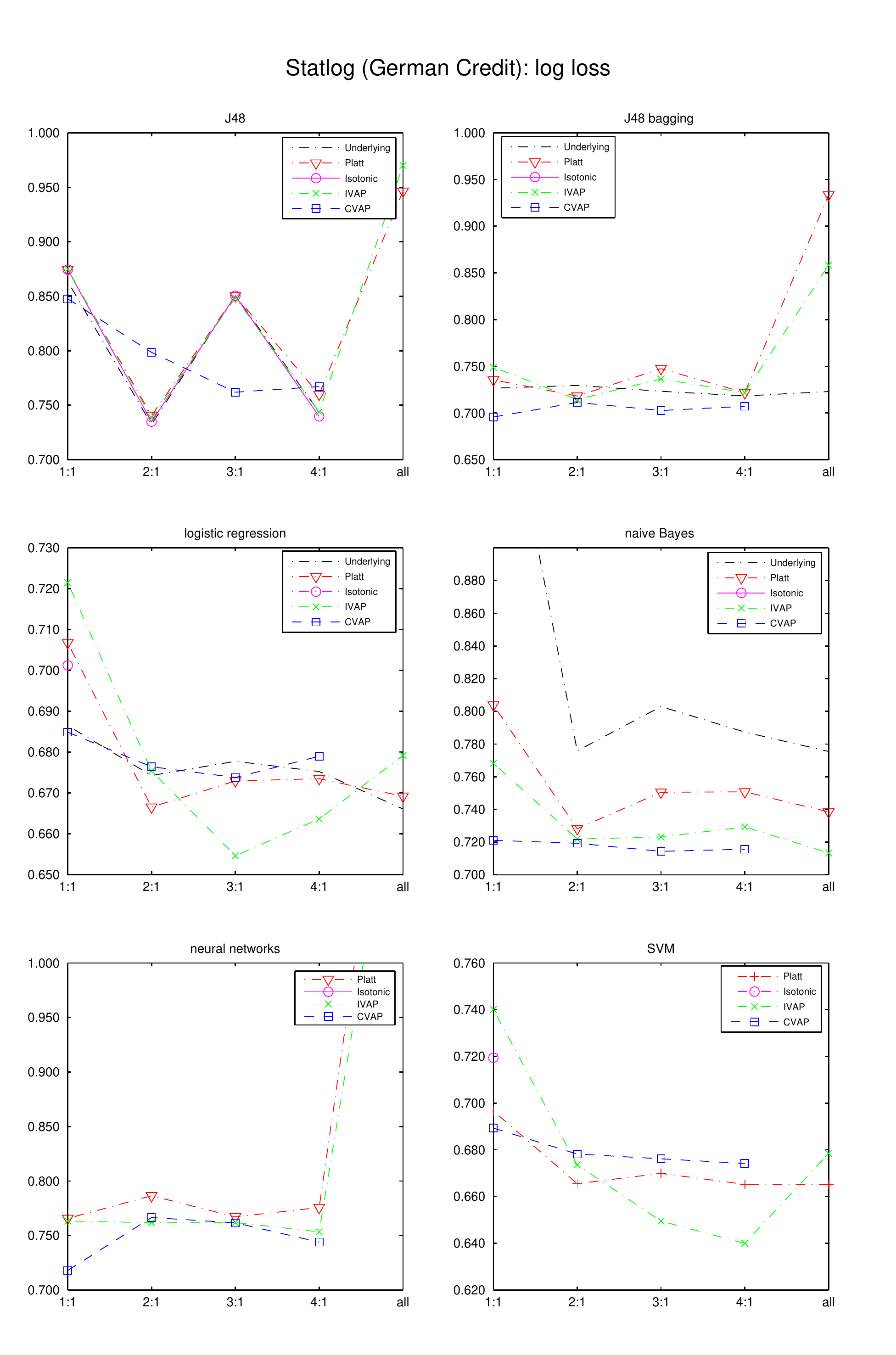}
    \end{center}
    \caption{The analogue of Figure~\ref{fig:log-adult} for the data set \texttt{Statlog German Credit Data}.}
    \label{fig:log-Statlog}
  \end{figure*}

  \begin{figure*} 
    \begin{center}
      \includegraphics[trim = 0mm 10mm 0mm 12mm, clip, width=\textwidth]{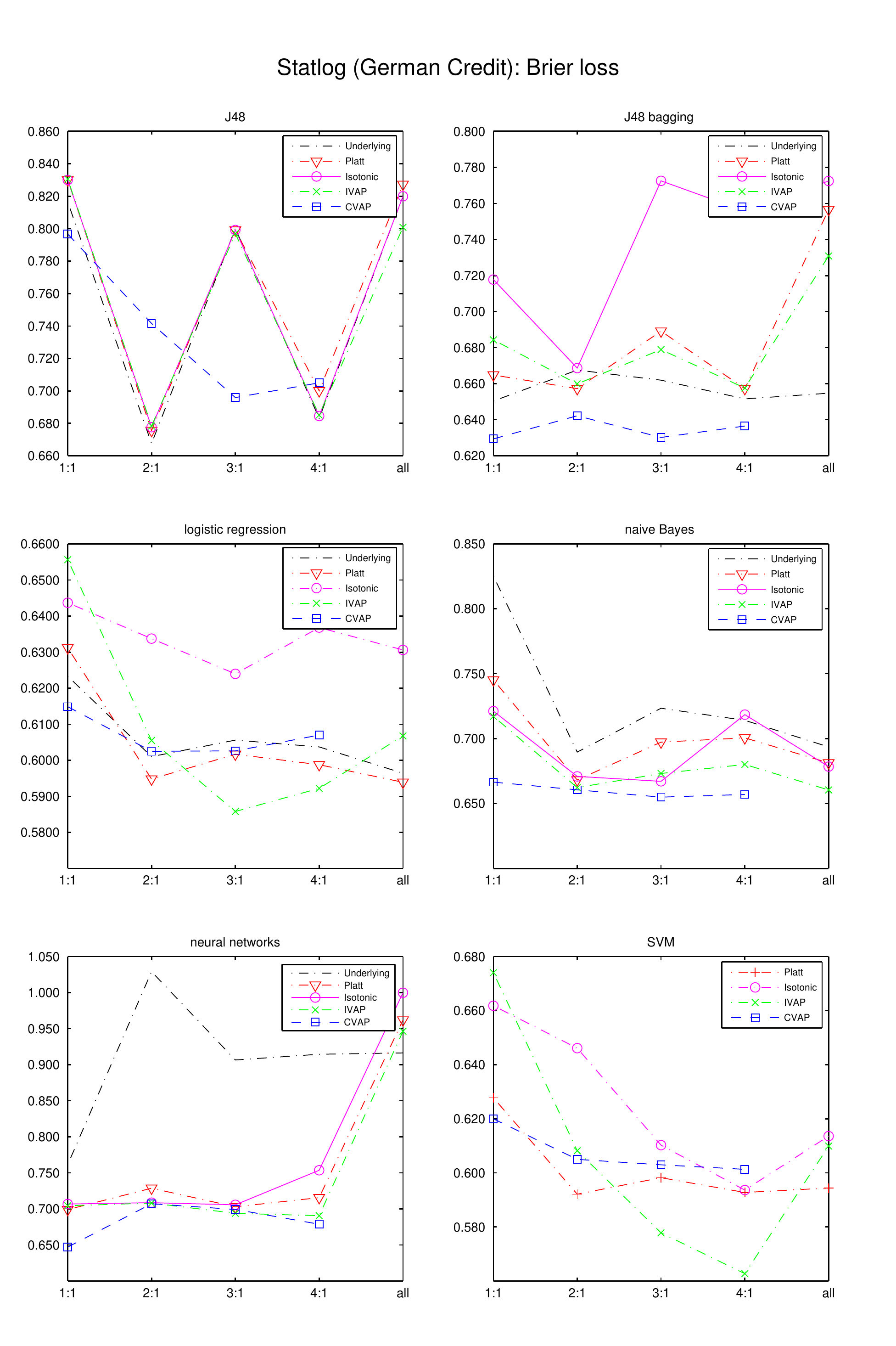}
    \end{center}
    \caption{The analogue of Figure~\ref{fig:Brier-adult} for the data set \texttt{Statlog German Credit Data}.}
    \label{fig:Brier-Statlog}
  \end{figure*}
\fi

And finally, Figures~\ref{fig:log-adult10} and~\ref{fig:Brier-adult10}
show the results for log loss and Brier loss, respectively, for the \texttt{adult} data set
and for a wide range of the ratios of the size of the proper training set to the calibration set.
The left-most column of each plot is $1:9$, which means,
in the case of Platt's method, isotonic regression, and IVAPs,
that 10\% of the training set was allocated to the proper training set and the rest to the calibration set.
In the case of CVAPs, $1:9$ means that the training set was split into 10 folds,
each of them in turn was used as the proper training set, and the rest were used as the calibration set;
the results were merged using the minimax procedure as described in Section~\ref{sec:merging}.
In the case of the underlying algorithm, $1:9$ means that only 10\% of the training set
was in fact used for training (the same 10\% as for the first three calibration methods).
The other columns are $1:8$, $1:7$,\ldots, $1:2$,
$1:1$ (which corresponds to $1:1$ in
\ifPortrait Figures~\ref{fig:log-adult} and~\ref{fig:Brier-adult}\fi),\ldots,
$4:1$ (which corresponds to $4:1$ in
\ifPortrait Figures~\ref{fig:log-adult} and~\ref{fig:Brier-adult}\fi,
i.e., to the standard procedure of 5-fold cross-validation),
$5:1$,\ldots, $9:1$ (the latter corresponds to the other standard cross-validation procedure,
that of 10-fold cross-validation);
the results in those columns are analogous to those in the column $1:9$.
In order not to duplicate the information we gave earlier for the \texttt{adult} data set,
we give the results for a randomly permuted \texttt{adult} data set.
There is not much difference between 5 and 10 folds for most underlying algorithms
(logistic regression behaves unusually in that its performance deteriorates
as the size of the proper training set increases,
perhaps because less data are available for calibration).

\begin{figure*} 
  \begin{center}
    \includegraphics[trim = 0mm 10mm 0mm 12mm, clip, width=\textwidth]{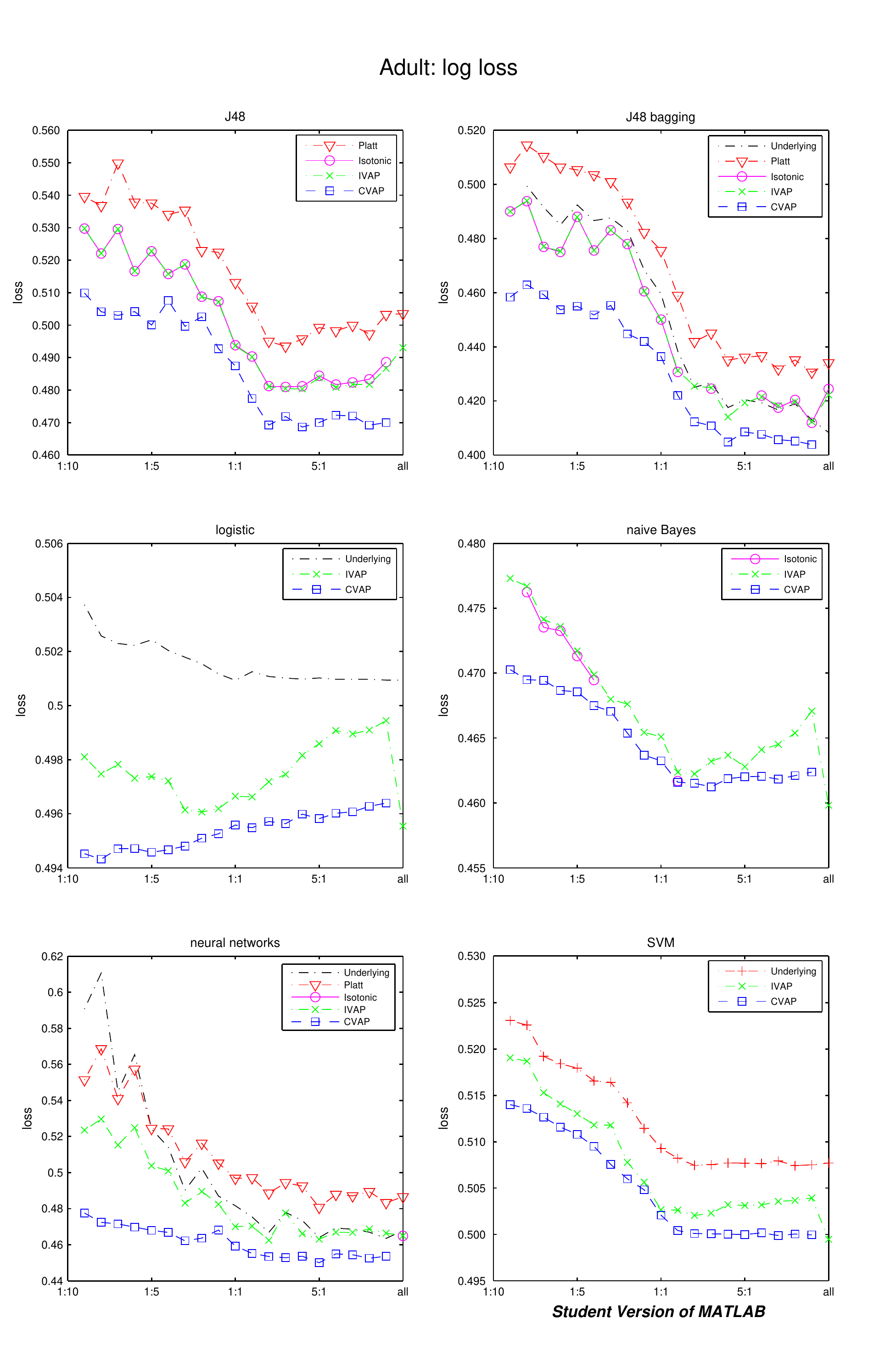}
  \end{center}
  \caption{The log loss on the \texttt{adult} data set of the six prediction algorithms
    and four calibration methods}
  \label{fig:log-adult10}
\end{figure*}

\begin{figure*} 
  \begin{center}
    \includegraphics[trim = 0mm 10mm 0mm 12mm, clip, width=\textwidth]{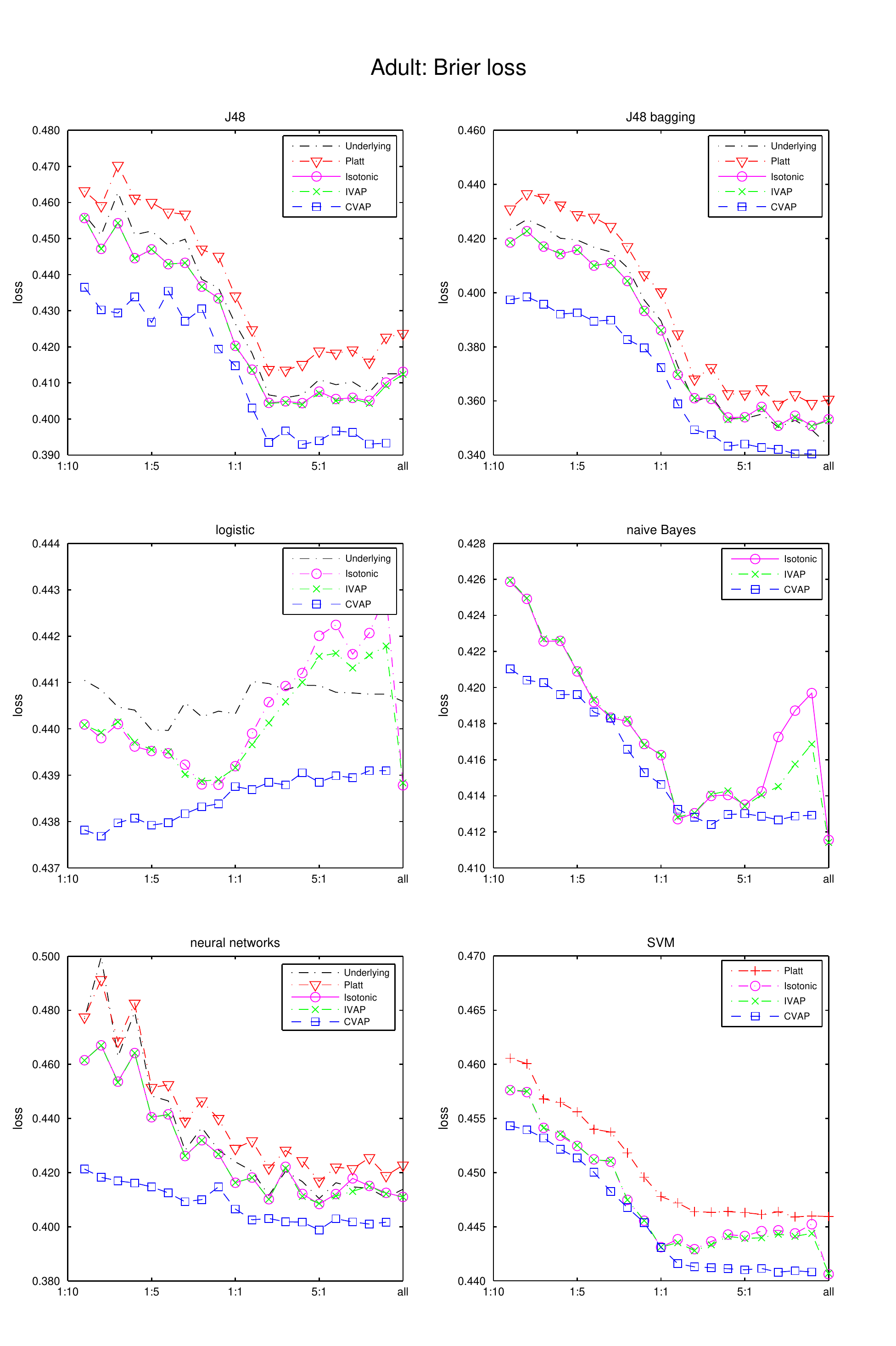}
  \end{center}
  \caption{The analogue of Figure~\ref{fig:log-adult10} for the Brier loss function}
  \label{fig:Brier-adult10}
\end{figure*}

\section{Conclusion}

This paper introduces two new computationally efficient algorithms for probabilistic prediction,
IVAP, which can be regarded as a regularised form of the calibration method
based on isotonic regression,
and CVAP, which is built on top of IVAP using the idea of cross-validation.
Whereas IVAPs are automatically perfectly calibrated,
the advantage of CVAPs is in their good empirical performance.


This paper does not study empirically upper and lower probabilities produced by IVAPs and CVAPs,
whereas the distance between them provides information about the reliability
of the merged probability prediction.
Finding interesting ways of using this extra information is one of the directions of further research.

\subsection*{Acknowledgments}

We are grateful to the conference reviewers for numerous helpful comments and observations,
to Vladimir Vapnik for sharing his ideas about exploiting synergy between different learning algorithms,
and to participants in the conference \emph{Machine Learning: Prospects and Applications}
(October 2015, Berlin) for their questions and comments.
The first author has been partially supported by EPSRC (grant EP/K033344/1)
and AFOSR (grant ``Semantic Completions'').
The second and third authors are grateful to their home institutions for funding their trips to Montr\'eal\ifnotCONF\
  to attend NIPS 2015\fi.

\end{document}